\newif\ifdraft \drafttrue

\documentclass[11pt]{article}
\usepackage[title]{appendix}
\ifdraft
\else
\usepackage{icml2021}
\usepackage{caption}
\fi

\usepackage{microtype}
\usepackage{graphicx}
\usepackage{subfigure}
\usepackage{booktabs} 
\usepackage{amssymb,amsmath,amsthm,amsfonts}
\usepackage{tikz}

\usepackage{hyperref}

\ifdraft

\usepackage{geometry}
\else
\fi

\usepackage[utf8]{inputenc} 
\usepackage{hyperref}       
\usepackage{url}            
\usepackage{booktabs}       
\usepackage{amsfonts}       
\usepackage{nicefrac}       
\usepackage{microtype}      

\usepackage{thmtools,bm}
\usepackage{mathtools,algorithm}
\usepackage{jinshuomacros}
\usepackage[noend]{algpseudocode}
\usepackage{algorithmicx}
\usepackage{mathrsfs}
\usepackage{graphicx}
\usepackage{bibentry}
\usepackage{environ}
\usepackage{thm-restate}

\nobibliography* 

\declaretheoremstyle[bodyfont=\normalfont]{normalbody}
\newtheorem{theorem}{Theorem}[section]
\newtheorem{lemma}[theorem]{Lemma}

\newtheorem{corollary}[theorem]{Corollary}
\newtheorem{proposition}[theorem]{Proposition}

\declaretheorem[style=normalbody]{remark}
\newtheorem{definition}[theorem]{Definition}
\newtheorem{conjecture}[theorem]{Conjecture}

\usepackage{amssymb}
\usepackage{tcolorbox}
\usepackage{cleveref}

\newcommand{\tr}{\mathrm{Tr}\,}
\newcommand{\Proj}{P}
\newcommand{\Fisher}{\mathcal{I}}

\newcommand{\goto}{\rightarrow}
\DeclareMathOperator{\sign}{sgn}

\newcommand{\djs}[1]{\ifdraft \textcolor{blue}{[Jinshuo: #1]}\fi}

\ifdraft
\title{A Central Limit Theorem for Differentially Private Query Answering}
\author{
Jinshuo Dong\thanks{Northwestern University and the Institute for Data, Econometrics, Algorithms, and Learning (IDEAL).}
\and Weijie J.~Su\thanks{Department of Statistics, University of Pennsylvania.}
\and Linjun Zhang\thanks{Department of Statistics, Rutgers University.}}
\else
\icmltitlerunning{A Central Limit Theorem for Differentially Private Query Answering}
\fi

\begin{document}

\ifdraft
\maketitle
\newcommand{\citet}{\cite}
\newcommand{\citep}{\cite}
\fi

\begin{abstract}
Perhaps the single most important use case for differential privacy is to privately answer numerical queries, which is usually achieved by adding noise to the answer vector. The central question, therefore, is to understand which noise distribution optimizes the privacy-accuracy trade-off, especially when the dimension of the answer vector is high. Accordingly, extensive literature has been dedicated to the question and the upper and lower bounds have been matched up to constant factors \citep{bun2018fingerprinting,SteinkeUl17}. In this paper, we take a novel approach to address this important optimality question. We first demonstrate an intriguing central limit theorem phenomenon in the high-dimensional regime. More precisely, we prove that a mechanism is approximately \emph{Gaussian Differentially Private} \citep{dong2019gaussian} if the added noise satisfies certain conditions. In particular, densities proportional to $\e^{-\|x\|_p^\alpha}$, where $\|x\|_p$ is the standard
$\ell_p$-norm, satisfies the conditions. Taking this perspective, we make use of the Cramer--Rao inequality and show an ``uncertainty principle''-style result: the product of the privacy parameter and the $\ell_2$-loss of the mechanism is lower bounded by the dimension. Furthermore, the Gaussian mechanism achieves the constant-sharp optimal privacy-accuracy trade-off among all such noises. Our findings are corroborated by numerical experiments.
\end{abstract}

%

\newcommand{\err}{\mathrm{err}}
\section{Introduction} 
\label{sec:introduction}


Introduced in \citet{DMNS06}, \textit{differential privacy} (DP) is perhaps the most popular privacy definition. In addition to a rich academic
literature, differential privacy is now being deployed on a large scale by 
Apple \cite{apple},
Google \cite{rappor,bittau2017prochlo},
Uber \cite{JohnsonNeSo18},
LinkedIn \cite{rogers2020linkedin,expcomposition}
and the US Census Bureau \cite{DajaniLaSiKiReMaGaDaGrKaKiLeScSeViAb17}.
One of the most important and successful applications of DP is to answer numeric queries. Given a function $f$ of interest, which is also termed a query, our goal is to evaluate this (potentially vector-valued) query $f$ on the sensitive data. To preserve privacy, a DP mechanism $M$ working on a dataset $D$, in its simplest form, is defined as
\begin{equation}\label{eq:noise}
	M(D)=f(D)+tX.
\end{equation}
Above, $X$ denotes the noise term and $t$ is a scalar, which together are selected depending on the properties of the query $f$ and the desired privacy level. Among these, perhaps the most popular examples are the Laplace mechanism and the Gaussian mechanism where the noise $X$ follows the Laplace distribution and the Gaussian distribution, respectively.

Aside from privacy considerations, the most important criterion of an algorithm is arguably the estimation accuracy in the face of choosing, for example, between the Laplace mechanism or its Gaussian counterpart for a given problem. To be concrete, consider a real-valued query $f$ with sensitivity 1---that is, $\Delta f=\sup_{D,D'}|f(D)-f(D')|=1$, where the supremum is over all neighboring datasets $D$ and $D'$. Assuming $(\ep, 0)$-DP for the mechanism $M$, we are interested in minimizing its $\ell_2$ loss defined as
$$\err(M):=\E(M(D)-f(D))^2 = \E(tX)^2 = t^2\E X^2.$$
This question is commonly\footnote{If $f$ is integer-valued, then the doubly geometric distribution is a better choice and yields an $\ell_2$-loss of $\frac{1}{2\sinh^2\frac{\ep}{2}}<\frac{2}{\ep^2}$. In the so-called high privacy regime, i.e., $\ep\to0$, the two $\ell_2$-losses have the same order in the sense that their ratio goes to 1.} addressed by setting $X$ as a standard Laplace random variable and $t =\ep^{-1}$~\citep{DMNS06}. This gives $\err(M) = \frac{2}{\ep^2}$. Moving forward, we \textit{relax} the privacy constraint from $(\ep, 0)$-DP to $(\ep,\delta)$-DP for some small $\delta$. The canonical way, which was born together with the notion of $(\ep,\delta)$-DP, is to add Gaussian noise \citep{approxdp}. A well-known result demonstrates that Gaussian mechanism with $X$ being the standard normal and $t=\frac{1}{\ep}\sqrt{2\log(1.25\delta^{-1})}$ is $(\ep,\delta)$-DP~(see, e.g., \citet{DworkRo14}). The $\ell_2$-loss is $\err(M) = t^2=\frac{2\log(1.25\delta^{-1})}{\ep^2}$. 

A quick comparison between the two errors reveals a surprising message. The latter error $\frac{2\log(1.25\delta^{-1})}{\ep^2}$ is larger than the former $\frac{2}{\ep^2}$. In fact, the extra factor $\log(1.25\delta^{-1})$ is already greater than 10 when $\delta=10^{-5}$. At least on the surface, this observation is not consistent with the fact that $(\ep,\delta)$-DP is a relaxation of $(\ep,0)$-DP. Put differently, the Gaussian mechanism gives us a looser privacy guarantee while degrading\footnotemark~ the $\ell_2$ estimation accuracy. Instead of refuting the notion of $(\ep,\delta)$-DP, nevertheless, this consistency suggests that we need a better mechanism than the Gaussian mechanism, at least for \textit{one-dimensional} query-answering. Indeed, the truncated Laplace mechanism has been proposed as a better alternative to achieve $(\ep,\delta)$-DP \citep{geng2020tight}, which outperforms the Laplace mechanism in terms of estimation accuracy.  

\footnotetext{One may blame the sub-optimality of the choice of $t$, but the problem remains even if the smallest possible $t$ from \citet{balle2018improving} is applied. See \Cref{sec:preliminary} and the appendix for further details.}

\begin{figure*}[t]
    \centering
    \includegraphics[width=\textwidth]{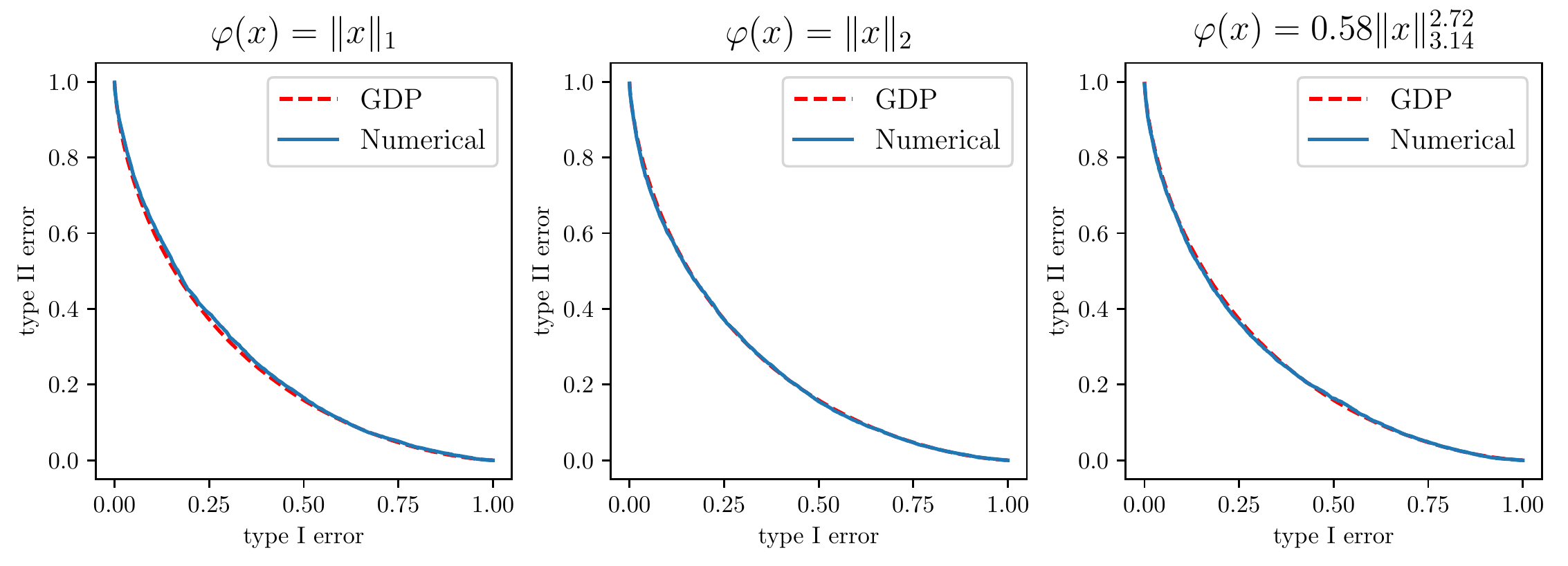}
    \vskip-5pt
    \caption{Fast convergence to Gaussian Differential Privacy (GDP) as claimed in \Cref{thm:informal}. Blue solid curves indicate the true privacy (i.e., ROC functions, see \Cref{sec:preliminary} for details) of the noise-addition mechanism considered in \Cref{thm:informal}. Red dashed curves are GDP limit predicted by our CLT. In all three panels the dimension $n=30$. Numerical details can be found in the appendix.}
    \label{fig:N}
\end{figure*}


Motivated by these facts concerning the Laplace, Gaussian, and truncated Laplace mechanisms, one cannot help asking:
\begin{enumerate}[label=(Q\arabic*)]
    \item\label{q1} Are there any insights behind the design of mechanisms in the one-dimensional setting that better utilizes the privacy budget? In particular, why was the truncated Laplace mechanism not considered in the first place?
	\item\label{q3}   More importantly, do the insights gained from \ref{q1} extend to private query-answering in high dimensions?
\end{enumerate}
In this paper, we tackle these fundamental questions, beginning with explaining \ref{q1} in Section~\ref{sec:preliminary} from the decision-theoretic perspective of DP~\citep{wasserman_zhou,KOV,dong2019gaussian}. However, our main focus is \ref{q3}. In addressing this question, we uncover a seemingly surprising phenomenon --- it is impossible to utilize the $(\ep,\delta)$ privacy budget in high-dimensional problems the same way as the truncated Laplace mechanism utilizes it in the one-dimensional problem. More specifically, we show a central limit behavior of the noise-addition mechanism in high dimensions, which, roughly speaking, says that for general noise distributions, the corresponding mechanisms \textit{all} behave like a Gaussian mechanism. The formal language of ``a mechanism behaves like the Gaussian mechanism'' was set up in \citet{dong2019gaussian}, where a notion called \textit{Gaussian Differential Privacy} (GDP) was proposed. Roughly speaking, a mechanism is $\mu$-GDP if it offers as much privacy as adding $N(0,\mu^{-2})$ noise to a sensitivity-1 query. As in the $(\ep,\delta)$-DP case, the smaller $\mu$ is, the stronger privacy is offered.


To state our first main contribution, let $f$ be an $n$-dimensional query and assume that its $\ell_2$-sensitivity is 1. Consider the noise-addition mechanism $M(D)=f(D)+tX$ where $X$ has a log-concave density $\propto \e^{-\varphi(x)}$ on $\R^n$. Let $\Fisher_X:=\E[\nabla\varphi(X)\nabla\varphi(X)^T]$ be the $n\times n$ Fisher information matrix and $\|\Fisher_X\|_2$ be its operator norm.
\begin{theorem}[Central Limit Theorem (Informal version of Theorem~\ref{thm:main})] \label{thm:informal}
Under certain conditions on $\varphi$, for $t={\mu}^{-1}\cdot \sqrt{\|\Fisher_X\|_2}$, the corresponding noise-addition mechanism $M$ defined in Eq.\eqref{eq:noise} is asymptotically $\mu$-GDP as the dimension $n\to\infty$ except for an $o(1)$ fraction of directions of $f(D)-f(D')$.
\end{theorem}
In particular, the norm power functions $\varphi(x)=\|x\|_p^\alpha$ $(p, \alpha\ge 1)$ satisfy these technical conditions. In fact, the convergence is very fast in this case. See \Cref{fig:N} for the numerical results.
We then elaborate on the condition ``$o(1)$ fraction of $f(D)-f(D')$''. Following the original definition, DP or GDP is a condition that needs to hold for arbitrary neighboring datasets $D$ and $D'$. This worst case perspective is exactly what prevents us to observe the central limit behavior. For example, consider a certain pair of datasets with $f(D)=(0,0,\ldots,0)$ and $f(D')=(1,0,\ldots,0)$, then privacy is completely determined by the first marginal distribution of $X$, and the dimension $n$ plays no role here. The ``$o(1)$ fraction of $f(D)-f(D')$'' rules out the essentially low-dimension cases and reveals the truly high-dimensional behavior.


In summary, \Cref{thm:informal} suggests that when the dimension is high, a large class of noise-addition mechanisms behave like the Gaussian mechanism, and hence are doomed to a poor use of the given $(\ep,\delta)$ privacy budget, in the same fashion as we have seen in the one-dimensional example.

However, admitting the central limit phenomenon, our second theorem turns the table and characterizes the optimal privacy-accuracy trade-off and justifies the Gaussian mechanism. To see this, recall that the noise-addition mechanism defined in \Cref{eq:noise} is determined by the pair $(t,X)$. Both privacy and accuracy are jointly determined by $t$ and $X$. Adopting the central limit theorem \ref{thm:informal}, it is convenient to take an equivalent parametrization, which is $(\mu,X)$, where $\mu$ is the desired (asymptotic) GDP parameter. Given $X$, the two parametrizations are related by $t={\mu}^{-1}\cdot \sqrt{\|\Fisher_X\|_2}$. Using parameters $(\mu,X)$, the corresponding mechanism $M_{\mu,X}$ is given by
\[
	M_{\mu,X}(D)=f(D)+{\mu}^{-1}\cdot \sqrt{\|\Fisher_X\|_2}\cdot X.
\]
Note that this definition implicitly assumes that the Fisher information of $X$ is finite.

By \Cref{thm:informal}, it is asymptotically $\mu$-GDP. The following theorem states in an ``uncertainty principle'' fashion that the privacy parameter and the error cannot be small at the same time.
\begin{theorem}\label{thm:l2}
	For all noise-addition mechanisms $M_{\mu,X}$ defined as above, we have
	$$\mu^2\cdot \err(M_{\mu,X})\geqslant n.$$
	The equality holds if $X$ is $n$-dimensional standard Gaussian.
\end{theorem}

Combining \Cref{thm:informal,thm:l2}, among all the noise that satisfies the conditions of \Cref{thm:informal}, Gaussian yields the constant-sharp optimal privacy-accuracy trade-off. As far as we know, this is the first result characterizing optimality with the sharp constant when the dimension is high.

The privacy conclusion of \Cref{thm:informal} does not work for every pair of neighboring datasets, so it is worth noting that we do NOT intend to suggest this as a valid privacy guarantee. Instead, we present it as an interesting phenomenon that has been largely overlooked in the literature. Furthermore, this central limit theorem admits an elegant characterization of privacy-accuracy trade-off that is sharp in constant. From a theoretical point of view, the proof of \Cref{thm:informal}, as we shall see in later sections, involves \textit{non-linear} functionals of high dimensional distributions. This type of results are, to the best of our knowledge, quite underexplored compared to linear functionals, so our results may serve as an additional motivation to study this type of questions.
\paragraph{Related work} 
\label{par:paragraph_name}

In the search of the optimal query-answering algorithm, the first step is to delimit the possible queries and the permissible algorithms. Specifically, let $\mathcal{F}$ be the set of possible queries and $\mathcal{M}$ be the set of permissible (differentially private) algorithms. A general mechanism $M\in\mathcal{M}$ maps the given query $f$ and dataset $D$ to an answer vector. Its incurred mean-squared error is
$$\err_{\mathcal{F}}(M):=\sup_{f\in \mathcal{F}}\sup_D\E \|M(f,D)-f(D)\|_2^2.$$
Note that this notion is consistent with the error previously defined for noise-addition mechanisms.

We look for an (approximate) error-minimizing mechanism in $\mathcal{M}$, that is, a mechanism $M_*\in\mathcal{M}$ such that
$$\err_{\mathcal{F}}(M_*)\approx \inf_{M\in \mathcal{M}} \err_{\mathcal{F}}(M) .$$
We expect different answers for different classes of queries  $\mathcal{F}$ and mechanisms $\mathcal{M}$. In fact, we have the following table. The references column is not intended as a complete list of relevant works.

\newcommand{\litrow}[3]{ #1 & #2 & #3 \\ \hline}
\newcommand{\specialcell}[2][c]{%
  \begin{tabular}[#1]{@{}c@{}}#2\end{tabular}}
\begin{table*}[!t]
	\label{tab:literature}
	\centering
	\begin{tabular}{|c|c|c|}
	\hline
	\litrow{References}{$\mathcal{F}$}{$\mathcal{M}$}
	\litrow{\cite{bun2018fingerprinting,SteinkeUl17}}{\{linear queries\}}{\{any DP algorithms\}}
	\litrow{\cite{nikolov2016geometry,edmonds2020power}}{\specialcell[t]{\{linear queries with\\bounded factorization norm\}}}{\{any DP algorithms\}}
	\litrow{\cite{chaudhuri2011differentially,bassily2014private}}{\{optimization  queries\}}{\{any DP algorithms\}}
	\litrow{\cite{cai2019cost,cai2020cost}}{\{regression queries\}}{\{any DP algorithms\}}
	\litrow{\cite{geng2014optimal,geng2020tight}}{\{$\R$ or $\R^2$ valued $f$\}}{\{any DP algorithms\}}
	\litrow{This work}{\{any $f$\}}{\{DP noise-addition algorithms\}}
	\end{tabular}
	\caption{All $\mathcal{F}$ are implicitly assumed to have bounded sensitivity.}
\end{table*}


The main difference between the current work and the majority of existing literature is that, informally, we pick a large $\mathcal{F}$ and a small $\mathcal{M}$, while others study small $\mathcal{F}$ and large $\mathcal{M}$. The major advantage of this choice is that it admits a constant-sharp lower bound for high-dimensional problems. On the other hand, most existing results only characterize optimality up to a constant factor, while the constant factor is crucial to bring differential privacy into practice. The only exceptions in the table are \cite{geng2014optimal,geng2020tight}, with the limit of the query being one or two dimensional. In addition, albeit the great success of these lower bound with $\mathcal{M}=$\{any DP algorithms\}, the technique is often highly involved and raises the bar of further research. By picking a large $\mathcal{F}$ and a small $\mathcal{M}$, we explore in a new direction that potentially bypass these difficulties.

In the end, we remark that for linear queries, shrinking $\mathcal{M}$ from arbitrary DP algorithms to noise-addition algorithms at worst blows up the privacy parameters by a factor independent of the dimension \cite{bhaskara2012unconditional}.


\section{GDP and the ROC Functions}
\label{sec:preliminary}
The decision theoretic interpretation of DP was first proposed in \citet{wasserman_zhou} and then extended by \citet{KOV}. More recently, \citet{dong2019gaussian} systematically studied this perspective and developed various tools. In this section we take this perspective and introduce the basics of \citet{dong2019gaussian}. This will allow us to give an intuitive answer to \ref{q1}.

Suppose each individual's sensitive	information is an element in the abstract set $\mathcal{X}$. A dataset $D$ of $k$ people is then an element in $\mathcal{X}^k$. 
Let a randomized algorithm $M$ take a dataset as input and let $D$ and $D'$ be two neighboring datasets, i.e., they differ by one individual. Differential privacy seeks to limit the power of an adversary identifying the presence of an arbitrary individual in the dataset. That is, with the output as the observation, telling apart $D$ and $D'$ must be hard for the adversary. Decision theoretically, the quality of such an identification attack is measured by the errors it makes. The more error it is forced to make, the more privcacy $M$ provides.

To breach the privacy, the adversary performs the following hypothesis testing attack:
\[
    H_0: \text{output} \, \sim \underbrace{M(D)}_P
    \;\; \text{vs.} \;\;
    H_1: \text{output} \, \sim \underbrace{M(D')}_Q.
\vspace{-1em}
\]
By the random nature of $M$, $M(D)$ and $M(D')$ are two distributions. We emphasize this point by denoting them by $P$ and $Q$. The errors mentioned above are simply the probabilities confusing $D$ and $D'$, and are called type I and type II errors, respectively\footnote{For this specific testing problem, mistaking $D$ as $D'$ corresponds to type I and mistaking $D'$ as $D$ corresponds to type II. In general there is no need to worry about which is which, as neighboring relation is often symmetric.}.

\begin{figure*}[!t]
    \includegraphics[width=\textwidth]{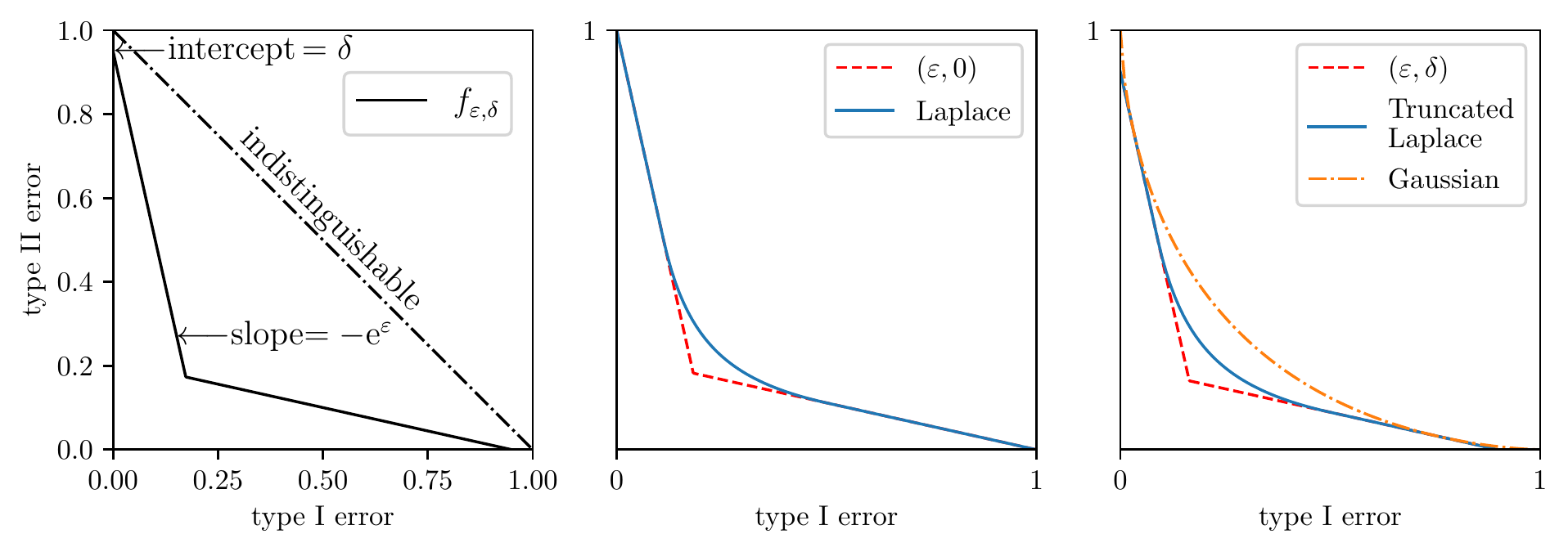}
    \vskip-5pt\caption{Left: $f_{\varepsilon,\delta}$ which recovers the classical $(\varepsilon,\delta)$-DP definition. Middle: Laplace mechanism is $(\varepsilon,0)$-DP Right: Gaussian mechanism and truncated Laplace mechanism are both $(\varepsilon,\delta)$-DP}
    \label{fig:prelim}
\end{figure*}

\paragraph{ROC function.}
For simplicity assume $M$ outputs a vector in $\R^n$. A general decision rule for testing $H_0$ against $H_1$ has the form $\phi:\R^n\to\{0,1\}$. Observing $v\in\R^n$, hypothesis $H_i$ is accepted if $\phi(v)=i$, for $i\in\{0,1\}$. The type I error of $\phi$, which is the probability of mistakenly accepting $H_1: v\sim {M(D')}=Q$ while actually $v\sim {M(D)}=P$, is $\alpha_\phi:=\P_{v\sim P}(\phi(v)=1)=\E_P(\phi)$. Similarly, the type II error of $\phi$ is $\beta_\phi:=1 - \E_Q(\phi)$. Note that both errors are in $[0,1]$.

Each test $\phi$ corresponds to a point $(\alpha_\phi,\beta_\phi)$ on the type I vs. type II plane (false positive vs. false negative plane). A family of such a test yields a (flipped) ROC curve. Among all families of tests, the family of optimal tests, which is determined by the Neyman--Pearson lemma, is of vital importance. Consider a function $f_{P,Q}:[0,1]\to[0,1]$ defined as follows:
\begin{equation*}
	f_{P, Q}(\alpha) := \inf \{1 - \E_Q(\phi): \phi \text{ satisfies } \E_P(\phi) \leqslant \alpha\}.
\end{equation*}
That is, $f_{P, Q}(\alpha)$ equals the minimum type II error that one can achieve at significance level $\alpha$. The graph of $f_{P, Q}$ is exactly the flipped ROC curve of the family of optimal tests. We call it the \textit{ROC function of the test $P$ vs. $Q$}. The same notion is called \textit{trade-off function of $P$ and $Q$} in \cite{dong2019gaussian} and is denoted by $T[P,Q]$. We avoid this name because in our paper ``trade-off'' mainly refers to the privacy-accuracy trade-off, but we will keep their notation.

\vspace{1em}
Plugging in the privacy context where $P=M(D),Q=M(D')$, from the discussion above, we see that $T[M(D),M(D')]$ measures the optimal error distinguishing $M(D)$ and $M(D')$. Therefore, a lower bound on $T[M(D),M(D')]$ implies privacy of $M$. Indeed, \citet{wasserman_zhou,KOV} showed that $M$ is $(\ep,\delta)$-DP if and only if $T[M(D),M(D')]\geqslant f_{\ep,\delta}$ pointwise in $[0,1]$ for any neighboring dataset $D,D'$. The graph of $f_{\ep,\delta}$ is plotted in the left panel of \Cref{fig:prelim}. Compared to a single $(\ep,\delta)$ bound, the ROC function $T[M(D),M(D')]$ provides a more refined picture of the privacy of $M$. In fact, \citet{dong2019gaussian} shows that the ROC function is equivalent to an infinite family of $(\ep,\delta)$ bounds, which is called the privacy profile in \citet{balle2020profile}.

Now we use the ROC function to answer \ref{q1} in the introduction. Namely, we want to explain the embarrasing failure of the Gaussian mechanism, and the success of truncated Laplace mechanism.

When $M$ is the Laplace mechanism, which is designed to be $(\ep,0)$-DP, it is not hard to determine $T[M(D),M(D')]$ via Neyman--Pearson lemma and verify that it is indeed lower bounded by $f_{\ep,0}$ (see the middle panel of \Cref{fig:prelim}). In fact, $T[M(D),M(D')]$ mostly agrees with $f_{\ep,0}$. In other words, the $(\ep,0)$ privacy budget is almost\footnote{If the query is interer-valued, then $(\ep,0)$ privacy budget can be saturated by adding doubly geometric noise.} fully utilized.

When $M$ is the Gaussian mechanism with $(\ep,\delta)$-DP gaurantee, $T[M(D),M(D')]$ is naturally lower bounded by $f_{\ep,\delta}$, however, there is a large gap between the two curves (see the right panel of \Cref{fig:prelim}). The $(\ep,\delta)$ privacy budget is poorly utilized by the Gaussian mechanism. This explains why the $l^2$-loss of the Gaussian mechanism is not satisfactory.

For a noise-addition mechanism, if the noise is sampled from the uniform distribution on $[-1,1]$ , then 
\[
T[M(D),M(D')]=f_{0,\delta}
\]
for some $\delta\in(0,1)$. This suggests that we should consider bounded noise if we want to add a $\delta$ slack in privacy to the Laplace mechanism. The obvious attempt is then to truncate the Laplace noise. Indeed, the corresponding ROC function is as close to $f_{\ep,\delta}$ as that of the Laplace mechanism to $f_{\ep,0}$ (also see the right panel of \Cref{fig:prelim}). This not only explains the success of the truncated Laplace mechanism, but also points us in the right direction to search for such a mechanism.

In hindsight, this achievement for one-dimensional mechanisms is due to the following fact: as we change the noise distribution, the corresponding ROC functions are significantly different. Hence we can pick the one that best utilizes our privacy budget. However, in the next section we will argue that this no longer works when the dimension is high --- many choices of noise distribution yield the same ROC function, which is the ROC of the Gaussian mechanism.

\paragraph{ROC function of the Gaussian mechanism} 
\label{par:roc_function_of_gaussian_mechanism}
For $\mu \ge 0$, let $G_\mu := T[\N(0,1), \N(\mu, 1)]$ where $\Phi$ denotes the cumulative distribution function (CDF) of the
standard normal distribution. Consider a query $f$ with sensitivity 1 and let $\mathrm{Lap}(0,1)$ be the standard Laplace noise. Just like $\ep$-DP captures the privacy of the mechanism $M(D)=f(D)+\ep^{-1}\cdot \mathrm{Lap}(0,1)$, the function $G_\mu$ captures the privacy of $M(D)=f(D)+\mu^{-1}\cdot N(0,1)$. In fact, if $f(D')-f(D)=1$, then $M(D)=N(f(D),\mu^{-2})$ and $M(D')=N(f(D),\mu^{-2})$. By its hypothesis testing construction, $T[P,Q]$ remains invariant when an invertible transformation is simultaneously applied to $P$ and $Q$, resulting in
\begin{align*}
	T[M(D),M(D')]
	&=T[N(f(D),\mu^{-2}),N(f(D'),\mu^{-2})]\\
	&=T[N(0,1),N(\mu,1)]=G_\mu.
\end{align*}
Therefore, the privacy of a Gaussian mechanism is precisely captured by the ROC function $G_\mu$. A general mechanism $M$ is said to be \emph{Gaussian differentially private} (GDP) if it offers more privacy than a Gaussian mechanism. More specifically,
\begin{definition}[GDP]\label{def:GDP}
 An algorithm $M$ is $\mu$-GDP if $T[M(D), M(D')] \ge G_\mu$ for any pair of neighboring datasets $D$ and $D'$.
\end{definition}
Alternatively, $M$ is $\mu$-GDP if and only if
\[
\inf_{D,D'}T[M(D), M(D')]\geqslant G_\mu,
\]
where the infimum of ROC functions is interpreted pointwise, and the infimum is over all neighboring datasets $D$ and $D'$. This inequality says $M$ offers more privacy than the corresponding Gaussian mechanism.
If it holds with equality, i.e.,
\begin{equation}\label{eq:GDP}
    \inf_{D,D'}T[M(D), M(D')]=G_\mu,
\end{equation}
then it means the mechanism $M$ offers exactly the same amount of privacy as the corresponding Gaussian mechanism. In fact, the CLT to be presented in the next section has this flavor of conclusion.

\newcommand{\KS}{\mathrm{KS}}
\section{Central Limit Theorem} 
\label{sec:CLT}

In the following two sections we turn to addressing \ref{q3}. This section is dedicated to the rigorous form of the CLT and the discussion. The experience with the CLT for i.i.d.~random variables suggests the statement for the normalized special case is usually the most comprehensible. Therefore, we will state the normalized version as \Cref{thm:main} and derive the general case as \Cref{cor:main}, which is also the rigorous version of our informal \cref{thm:informal} mentioned in the introduction.


Consider an $n$-dimensional query $f:\mathcal{X}^k\to \R^n$. We assume it has $\ell_2$-sensitivity 1, i.e., $\sup_{D,D'}\|f(D)-f(D')\|_2=1$. Suppose $\varphi:\R^n\to\R$ is convex and $\e^{-\varphi}$ is integrable on $\R^n$. We can generate a log-concave random vector with density $\propto \e^{-\varphi(x)}$ and it will be denoted by $X_\varphi$. Define the function class
\[
\mathfrak{F}_n:=\{\varphi:\R^n\to\R\text{ convex}\mid \varphi(x)=\varphi(-x),\e^{-\varphi}\in L^1(\R^n), \E\|X_\varphi\|_2^2<+\infty,\E\|\nabla\varphi(X_\varphi)\|_2^2<+\infty\}.
\]
The regularity conditions guarantee that $X_\varphi$ has finite second moments and the Fisher information matrix is
defined as $\Fisher_{\varphi}=\E[\nabla\varphi(X_\varphi)\nabla\varphi(X_\varphi)^T]$. Furthermore, we also have $\E X_\varphi=0$ by symmetry and $\E\nabla\varphi(X_\varphi)=0$ by the standard theory of the Fisher information\footnote
{	
	Whenever Fisher information is involved, the standard assumption is the \textit{quadratic mean differentiability} (QMD) condition, which implies various nice properties including $\E\nabla\varphi(X_\varphi)=0$. QMD holds for all log-concave location families as long as the integral $\E\|\nabla\varphi(X_\varphi)\|_2^2$ is finite. The proof of this claim is a straightforward consequence of Lemma 7.6 of \cite{van2000asymptotic}.
}. We will focus on this class of functions for the rest of the paper.

The $n$-dimensional noise-addition mechanism of interest takes the form $M(D)=f(D)+t X_\varphi$. The parameter $t$ is only for the convenience of tuning and can be absorbed into $\varphi$. In fact, $tX_\varphi$ has log-concave density $\propto \e^{-\varphi(x/t)}$, so it is distributed as $X_{\tilde{\varphi}}$ where $\tilde{\varphi}(x)=\varphi(x/t)$. 
For the normalized CLT, we set $t=1$ and assume $\Fisher_\varphi$ is the $n\times n$ identity matrix $I_{n\times n}$.

Since we are going to present an asymptotic result where the dimension $n\to\infty$, the above objects necessarily appear with an index $n$, i.e., we have $f_n, \varphi_n$, $X_{\varphi_n}$ and $\Fisher_{\varphi_n}$. The latter two are often denoted by $X_n$ and $\Fisher_n$ for brevity. With normalization, the $n$-dimensional mechanism of interest is
$M_n(D)=f_n(D)+ X_{n}$. For clarity, we choose to state the theorem first, and then present the details of the technical conditions.

\begin{restatable}[]{theorem}{main}
\label{thm:main}
	If the function sequence $\varphi_n$ satisfies conditions \ref{d1} and \ref{d2}, then there is a sequence of positive numbers $c_n$ with $c_n\to0$ as $n\to\infty$ and a subset $E_n\subseteq S^{n-1}$ with $\P_{v\sim S^{n-1}}(v\in E_n)>1-c_n$ such that 
	\[\|\inf_{D,D'} T[M_n(D),M_n(D')]-G_1\|_\infty\leqslant c_n,\]
	where the infimum is taken over $D,D'$ such that $\tfrac{f_n(D')-f_n(D)}{\|f_n(D')-f_n(D)\|_2}\in E_n$.
\end{restatable}

Here $v\sim S^{n-1}$ means $v$ comes from a uniform distribution of the unit sphere $S^{n-1}\subseteq \R^n$. The conclusion is basically that $\inf_{D,D'} T[M_n(D),M_n(D')]\to G_1$, i.e., $M_n$ is asymptotically GDP. Similar to the interpretation of \eqref{eq:GDP}, it means the mechanism $M_n$ provides the same amount of privacy as a Gaussian mechanism in the limit of $n\to\infty$. However, a fraction of neighboring datasets has to be excluded. More specifically, the limit holds if the direction of the difference $f_n(D')-f_n(D)$ falls in $E_n$, an ``almost sure'' event as the dimension $n\to\infty$. As we remarked in the introduction, directions in $E_n$ can exhibit low dimensional behavior and hence must be ruled out for any high-dimensional observation.

For a vector $v\in\R^n$ and $\varphi\in\mathfrak{F}_n$, let $\Proj_v^\varphi:\R^n\to\R$ be defined as $\Proj_v^\varphi(x)=\varphi(x+v)-\varphi(x)-\frac{1}{2}v^T\Fisher_\varphi v$. For two random variables $X$ and $Y$, their Kolmogorov--Smirnov distance $\KS(X,Y)$ is defined as the $\ell_\infty$ distance of their CDFs. A sequence of random variables is denoted by $o_{P}(1)$ if they converge in probability to 0. The technical conditions for the CLT are as follows. Note that each of them are conditions on the function sequence $\varphi_n$.
\begin{enumerate}[label=(D\arabic*)]
    \item\label{d1} $\KS(P_v^{\varphi_n}(X_n),v^T\nabla\varphi_n(X_n)=o(1)$ with probability at least $1-o(1)$ over $v\sim S^{n-1}$;
	\item\label{d2} $\|\nabla\varphi_n(X_n)\|_2=\sqrt{n}\cdot(1+o_{P}(1))$.
\end{enumerate}
\begin{remark}
	Dropping the cumbersome subscripts $n$, \ref{d1} is roughly concerned with
	$$P_v^\varphi(X)=\varphi(X+v)-\varphi(X)-\tfrac{1}{2}v^T\Fisher_\varphi v\approx v^T\nabla\varphi(X).$$
	Since $\Fisher_\varphi$ is the expectation of the Hessian of $\varphi$, we see that \ref{d1} is basically a regularity condition stating that the Taylor expansion of $\varphi$ holds on average up to the second order.
\end{remark}
\begin{remark}
Condition \ref{d2} basically says that $\nabla\varphi(X)$ mostly falls on a spherical shell of radius $\sqrt{n}$ (as it should since $\Fisher_\varphi=\E[\nabla\varphi(X)\nabla\varphi(X)^T]$ is assumed to be identity). A deeper understanding is provided by
an alternative interpretation of condition \ref{d1}, using a new notion we propose called ``likelihood projection.''
\end{remark}
\paragraph{Likelihood Projection.} The function $\Proj_v^\varphi$ defined above is called the ``likelihood projection'' along direction $v$. It is (up to an additive constant) the log likelihood ratio of $X_\varphi$ and its translation $X_\varphi-v$. In fact, $X_\varphi$ has density $\frac{1}{Z_\varphi}\e^{-\varphi(x)}$ and $X_\varphi-v$ has density $\frac{1}{Z_\varphi}\e^{-\varphi(x+v)}$ where $Z_\varphi$ is the common normalizing constant. The log likelihood ratio is $\varphi(x+v)-\varphi(x)$. This explains the word ``likelihood.'' To observe its nature as a ``projection,'' consider the special case $\varphi(x)=\frac{1}{2}\|x\|_2^2$. Straightforward calculation suggests that $\Fisher_\varphi$ is identity and $\Proj_v^\varphi(x)=v^Tx$. So it is indeed a generalization of the linear projection along direction $v$.

The alternative interpretation of condition \ref{d1} is that when the dimension is high, the ``likelihood projection'' $P_v^\varphi(X)$ is roughly a linear projection to the direction $v$. Condition \ref{d2} is then the ``thin-shell'' condition proposed in Sudakov's theorem \citep{sudakov1978typical}, which we state in the appendix as a necessary tool for the proof of our CLT.

For the general case, consider $M_n(D)=f_n(D)+t_n X_{\varphi_n}$ where
\begin{equation*}
	t_n={\mu^{-1}}\cdot \sqrt{\|\Fisher_{n}\|_2}.
\end{equation*}
The factor $\sqrt{\|\Fisher_{n}\|_2}$ normalizes the Fisher information to the identity, and the factor $\mu^{-1}$ controls the final privacy level.
For this mechanism, we have
\begin{corollary}\label{cor:main}
	If the function sequence $\tilde{\varphi}_n(x)=\varphi_n(\|\Fisher_{n}\|_2^{-\frac{1}{2}}x)$ satisfies conditions \ref{d1} and \ref{d2} and that $\Fisher_{n} = \|\Fisher_{n}\|_2\cdot(1+o(1))\cdot I_{n\times n}$, then there is a sequence of positive numbers $c_n\to0$ and a subset $E_n\subseteq S^{n-1}$ for each $n$ with $\P_{v\sim S^{n-1}}(v\in E_n)>1-c_n$ such that 
\[
\|\inf_{D,D'} T[M_n(D),M_n(D')]-G_\mu\|_\infty\leqslant c_n,
\]
where the infimum is taken over $D,D'$ such that $\tfrac{f_n(D')-f_n(D)}{\|f_n(D')-f_n(D)\|_2}\in E_n$.
\end{corollary}

In particular, we show that for $p\in[1,+\infty),\alpha\in[1,+\infty)$, norm powers $\|x\|_p^\alpha$ satisfy the above conditions.
\begin{lemma}\label{lem:normpower}
	For $p\in[1,+\infty),\alpha\in[1,+\infty)$, let $c_{p,\alpha}=\alpha^{-1}\cdot p^{-\alpha+\frac{\alpha}{p}}\cdot\left(\frac{\Gamma(2-\frac{1}{p})}{\Gamma(\frac{1}{p})}\right)^{-\frac{\alpha}{2}}$, the sequence of functions $\varphi_n(x)=n^{1-\frac{\alpha}{p}}\cdot c_{p,\alpha}\|x\|_p^\alpha$ satisfies conditions \ref{d1} and \ref{d2} and that $\Fisher_{n} = \|\Fisher_{n}\|_2\cdot(1+o(1))\cdot I_{n\times n}$.
\end{lemma}
The parameter $c_{p,\alpha}$ and the power of $n$ are determined by the Fisher information, which can be found in \Cref{thm:variance}.

More generally, we conjecture that
\begin{conjecture}\label{conj}
	All functions in $\mathfrak{F}_n$ satisfy \ref{d1} and \ref{d2}.
\end{conjecture}
Recall that $\varphi\in\mathfrak{F}_n$ lead to log-concave distributions. We limit the scope of our conjecture to log-concave distributions because of an interesting lemma involved in the proof of the central limit theorem \ref{thm:main}. Consider the mechanism $M^t(D)=f(D)+tX$, with the emphasis on the scaling parameter $t$. As $t$ increases, $M^t$ obviously loses accuracy no matter what distribution $X$ has. On the other hand, we have

\begin{restatable}[]{lemma}{monotone}
\label{lem:monotone}
When $X$ has log-concave distribution and $t\geqslant0$, the ROC function $T[M^t(D),M^t(D')]$ is (pointwise) monotone increasing in $t$ for any $D,D'$.
\end{restatable}
Since larger ROC function means more privacy, this lemma confirms that $M^t$ gains privacy as $t$ increases. In other words, it confirms the existence of a ``privacy-accuracy trade-off'' given the log-concavity of $X$.
Note that without log-concavity, monotonicity in the lemma need not hold. For a one-dimensional example, consider an $X$ that supports on even numbers and $f(D)=0,f(D')=2$. When $t=2$, $T[M^t(D),M^t(D')]=T[2X,2X+2]=T[X,X+1]$. There is no privacy in this case as $X$ and $X+1$ have completely disjoint support. On the other hand, when $t=1$, $T[M^t(D),M^t(D')]=T[X,X+2]$ and incurs some privacy. That is, more noise does not imply more privacy, hence violating the conclusion of \Cref{lem:monotone}.

In summary, results in this section show that mechanisms adding noise that satisfies \ref{d1} and \ref{d2} (e.g. densities $\propto \e^{-\|x\|_p^\alpha}$) behave like a Gaussian mechanism. Changing the noise in this class does not change the ROC function by much. Hence we cannot repeat the success at fully utilizing the $(\ep,\delta)$ privacy budget as in \Cref{sec:preliminary}.

On the other hand, our CLT involves Fisher information, and hence gives us the opportunity to relate to the (arguably) most successful tool for constant-sharp lower bound --- the Cramer--Rao inequality. This will be the content of the next section.

\section{Privacy-Accuracy Trade-off via Uncertainty Principles} 
\label{sub:uncertainty}
The central limit theorem in the previous section suggests that we use the GDP parameter $\mu$ to measure privacy. Adopting this, we will show that the  privacy-accuracy trade-off is naturally characterized by the Cramer--Rao lower bound. The conclusion has a similar flavor to the uncertainty principles.

%


Recall that the mechanism $M(D)=f(D)+tX_\varphi$ is determined by two ``parameters'': the shape parameter $\varphi\in \mathfrak{F}_n$, which determines the distribution of $X_\varphi$, and the scale parameter $t$. If $\varphi$ also satisfies the conditions of \Cref{thm:main}, then we can use the desired (asymptotic) GDP parameter $\mu$ to determine the scale parameter, i.e., $t={\mu}^{-1}\cdot \sqrt{\|\Fisher_\varphi\|_2}$. Using the equivalent parametrization $(\mu,\varphi)$, the corresponding mechanism $M_{\mu,\varphi}$ is given by
\begin{equation}\label{eq:mech}
	M_{\mu,\varphi}(D)=f(D)+{\mu}^{-1}\cdot \sqrt{\|\Fisher_\varphi\|_2}\cdot X_\varphi.
\end{equation}

\newcommand{\myrow}[6]{ #1 & #3 & #2 & #4 & #5 & #6 \\ \hline}
\begin{table*}[!t]
	\caption{Explicit expressions of Fisher information and mean squared error.
	}	\label{tab:compare}
	\centering
	\begin{tabular}{|c|c|c|c|c|c|}
	\hline
	\myrow{Density}{$\E\|X\|_2^2$}{$\|\Fisher_\varphi\|_2$}{$\E\|X\|_2^2\cdot\|\Fisher_\varphi\|_2$}{$\E\|X\|_\infty^2$}{$\E\|X\|_\infty^2\cdot\|\Fisher_\varphi\|_2$}
	\myrow{$\propto \e^{-\|x\|_1}$ }{ $2n$ }{ $1$ }{ $2n$}{$\sim (\log n)^2$}{$\sim (\log n)^2$}
	\myrow{$\propto \e^{-\|x\|_2}$ }{ $n(n+1)$ }{ $\frac{1}{n}$ }{ $n+1$}{$\sim 2n\log n$}{$\sim 2\log n$}
	\myrow{$\propto \e^{-\|x\|_2^2}$ }{ $\frac{1}{2}n$ }{ $2$ }{ $n$}{$\sim \log n$}{$\sim 2\log n$}
	\myrow{$\propto \e^{-\|x\|_p^\alpha}$ }{Lemma \ref{thm:variance}}{Lemma \ref{thm:variance}}{$\sim C_p\cdot n$}{Appendix}{$\leqslant C_p'\cdot(\log n)^{\frac{2}{p}}$}
	\end{tabular}
\end{table*}
As we have explained in the introduction, one way to measure the accuracy of the mechanism is the mean squared error of the noise
\begin{equation}\label{eq:err}
	\err(M_{\mu,\varphi})=\E\|tX_\varphi\|_2^2={\mu}^{-2}\cdot {\|\Fisher_\varphi\|_2}\cdot \E\|X_\varphi\|_2^2.
\end{equation}
The following theorem characterizes the privacy-accuracy trade-off as the product of the mean squared error $\err(M_{\mu,\varphi})$ and privacy parameter $\mu^2$.
\begin{theorem}[Restatement of \Cref{thm:l2}]\label{thm:CR}
If $\varphi\in\mathfrak{F}_n$, then for any mechanism $M_{\mu,\varphi}$ defined as in \eqref{eq:mech}, we have
	$$\mu^2\cdot \err(M_{\mu,\varphi})\geqslant n.$$
	In addition, the equality holds if the added noise $X$ is $n$-dimensional standard Gaussian.
\end{theorem}
\begin{proof}[Proof of \Cref{thm:CR}]
To simplify notations we will drop the subscript $\varphi$ in $X$.
By \eqref{eq:err}, we see that $\mu^2\cdot \err(M_{\mu,\varphi}) =  \E\|X\|_2^2\cdot {\|\Fisher_\varphi\|_2}$. So it suffices to show
\begin{equation}\label{eq:goal}
	\E\|X\|_2^2\cdot {\|\Fisher_\varphi\|_2}\geqslant n.
\end{equation}
We prove a slightly stronger result, namely
\begin{equation}\label{eq:UP}
	\E\|X\|_2^2\cdot {\tr\Fisher_\varphi}\geqslant n^2.
\end{equation}
\Cref{eq:goal} (and hence \Cref{thm:CR}) is a direct consequence of \eqref{eq:UP}. To see this, it suffices to show that $\|\Fisher_\varphi\|_2\geqslant \frac{1}{n}\tr\Fisher_\varphi$. Let $\lambda_1\geqslant\cdots\geqslant\lambda_n> 0$ be the eigenvalues of $\Fisher_\varphi$. Then
$$\|\Fisher_\varphi\|_2=\lambda_1\geqslant \tfrac{1}{n}(\lambda_1+\cdots+\lambda_n)=\tfrac{1}{n}\tr\Fisher_\varphi.$$
Next we focus on the proof of \eqref{eq:UP}. Consider the location family $\{X+\theta:\theta\in\R^n\}$. The Fisher information of this family is $\Fisher_\varphi$ at all $\theta$. The random vector itself is an unbiased estimator of the location. Therefore, by the Cramer--Rao inequality \citet{van2000asymptotic}, we have that $\mathrm{Cov}(X)- \Fisher_\varphi^{-1}$ is positive semi-definite. As a consequence,
\begin{align*}
\E\|X\|_2^2
		=\tr\E[XX^T]&\geqslant\tr\Fisher_\varphi^{-1}=\lambda_1^{-1}+\cdots+\lambda_n^{-1}.
\end{align*}
Therefore, by the Cauchy--Schwarz inequality,
	\begin{align*}
		\E\|X\|_2^2\cdot {\tr\Fisher_\varphi}&\geqslant
		(\lambda_1^{-1}+\cdots+\lambda_n^{-1})(\lambda_1+\cdots+\lambda_n)\geqslant n^2.
	\end{align*}
Thus, we prove \eqref{eq:UP} and obtain the desired result.
\end{proof}

We then would like to point out the similarity between \Cref{thm:CR} and the \textit{Uncertainty Principles}. Intuitively, they both claim that two desired quantities cannot be small simultaneously by showing a lower bound on their product: \Cref{thm:CR} claims that for noise-addition mechanisms, the mean squared error and the privacy parameter cannot be small simultaneously, while the (Hesenberg) uncertainty principle claims that for a physical system, the uncertainty in position and the uncertainty of momentum cannot be small simultaneously. A more precise comparision can be seen with the following presentation.

We abuse the notation $\Var[X]$ to denote the mean squared distance of a random vector $X\in\R^n$ from its expectation, i.e., $\Var[X] = \E[\|X-\E X\|_2^2]$.
For $\varphi\in\mathfrak{F}_n$, we have $\E X_\varphi=0$ and $\E\nabla\varphi(X_\varphi)=0$.
Therefore, $\Var[X] = \E\|X\|_2^2$ and $\Var[\nabla\varphi(X)] = \tr\,\Fisher_\varphi$. Hence \eqref{eq:UP} becomes
\begin{equation*}
		\Var [X_\varphi]\cdot\Var[\nabla\varphi(X_\varphi)]\geqslant n^2.
\end{equation*}
On the other hand, there are various mathematical manifestations of the uncertainty principle. The one behind the Hesenberg uncertainty principle is that a function and its Fourier transform cannot both be localized simultaneously. Specifically, for a function $f\in L^2(\R^n)$, its Fourier transform is defined as $\hat{f}(\xi) = \int\e^{-2\pi i\langle\xi,x\rangle}f(x)\diff x$. Fourier transform is unitary, i.e., $\|f\|_{L^2}=\|\hat{f}\|_{L^2}$. In particular, if $|f|^2$ is a probability density, then so is $|\hat{f}|^2$. Our previous abuse of notation also applies here, for example, $\Var[|f|^2]=\int(x-a)^T(x-a)|f(x)|^2\diff x$ where $a=\int x|f(x)|^2\diff x$.
For $\|f\|_{L^2}=1$, we have the following result\footnote{The Hessenberg uncertainty principle is a direct consequence of \Cref{eq:hesenberg} and the fact that the position operator and momentum operator are conjugate of each other via Fourier transform.}
, stated as Corollary 2.8 of \citep{folland1997uncertainty}.
\begin{align}\label{eq:hesenberg}
	\Var[|f|^2]\cdot \Var[|\hat{f}|^2]\geqslant\tfrac{n^2}{16\pi^2}.
\end{align}
This similarity suggests that \Cref{thm:CR} can be considered as yet another manifestation of the uncertainty principle.

Note that although \Cref{thm:CR} holds true for very general $\varphi$, the interpretation that $\mu$ is the asymptotic privacy parameter only holds for distributions that satisfy \ref{d1} and \ref{d2}. Therefore, let us consider the special case where $\varphi(x)=\|x\|_p^\alpha$. The corresponding $X_\varphi$ will be denoted by $X_{p,\alpha}$ and $\Fisher_\varphi$ by $\Fisher_{p,\alpha}$. In this special case, we can compute the quantities in \eqref{eq:goal} exactly. In the following lemma, we write $a_n\sim b_n$ for the two sequences $a_n$ and $b_n$ if $\frac{a_n}{b_n}\to1$ as $n\to\infty$.
\begin{restatable}[]{lemma}{var}
\label{thm:variance}
	For $1\leqslant p<\infty$ and $1\leqslant \alpha<\infty$, as $n\to\infty$, we have
\begin{align*}
	\E\|X_{p,\alpha}\|_2^2
	&\sim n^{\frac{2}{\alpha}-\frac{2}{p}+1} \cdot \alpha^{-\frac{2}{\alpha}}\cdot p^{\frac{2}{p}}\cdot\frac{\Gamma(\tfrac{3}{p})}{\Gamma(\frac{1}{p})}\\
	\Fisher_{p,\alpha}
	&\sim n^{\frac{2}{p}-\frac{2}{\alpha}}\cdot\alpha^{\frac{2}{\alpha}}\cdot p^{2-\frac{2}{p}}\cdot\frac{\Gamma(2-\frac{1}{p})}{\Gamma(\frac{1}{p})} \cdot  I_{n\times n}.
\end{align*}
\end{restatable}
This result put \Cref{thm:CR} into a more concrete context.
Some important cases with specific values of $p$ and $\alpha$ are worked out in \Cref{tab:compare}. Remarkably, in the last row, the products that characterize the privacy-accuracy trade-off are asymptotically independent of $\alpha$.
As a by-product of this calculation, we also derive the expression for the isotropic constant of the $n$-dimensional $\ell_p$ ball, which is an important concept in convex geometry. See the appendix for more results and discussion and \citep{brazitikos2014geometry} for more on the subject of convex geometry.

Alternatively, we may want to measure the accuracy by the expected squared $\ell_\infty$-norm of the noise. A similar argument suggests that we should consider the following quantity
$$\E\|X_{\varphi}\|_\infty^2\cdot\|\Fisher_{\varphi}\|_2.$$
By \eqref{eq:goal} and the fact that $\|x\|_\infty\geqslant\frac{1}{\sqrt{n}}\|x\|_2$, we have
\begin{equation}\label{eq:newCR}
\E\|X_{\varphi}\|_\infty^2\cdot\|\Fisher_{\varphi}\|_2
	\geqslant\tfrac{1}{n}
	\E\|X_{\varphi}\|_2^2\cdot\|\Fisher_{\varphi}\|_2\geqslant 1.
\end{equation}
We would like to point out a connection to a recently resolved open problem proposed in \citep{SteinkeUl17}, asking if there is a DP algorithm that answers a high-dimensional query with $\ell_2$-sensitivity 1 with $O(1)$ error in $\ell_\infty$ norm. In particular, recent solutions \citep{dagan2020bounded,ghazi2020avoiding} provide strong evidence that the lower bound in~\eqref{eq:newCR} is tight up to a constant factor.

\section{Conclusions and Future Work} 
\label{sec:conclusion}
In this work, we study constant-sharp optimality of noise-addition algorithms for high-dimensional query answering with differential privacy. We demonstrate that the ROC function offers good insight in the design of such algorithms when the dimension $n=1$. However, when $n$ is large, the CLT shows that a large class of noise-addition mechanisms behaves like a Gaussian mechanism through the lens of the ROC function. On the one hand, this suggests that $(\ep,\delta)$ privacy budget cannot be fully utilized, while on the other hand, if we adopt a GDP budget, then an ``uncertainty principle'' style of result yields an elegant characterization of the precise privacy-accuracy trade-off, and justifies the constant-sharp optimality of the Gaussian mechanism. We believe the insights offer a novel perspective to the long-lived privacy-accuracy trade-off question.

Various extensions are possible. An immediate one is to extend the CLT to a broader class of noise distributions, such as log-concave distributions as specified in Conjecture \ref{conj}. This may require significant machinary from convex geometry. For non-log-concave noise, \Cref{lem:monotone} suggests that a corresponding log-concave noise with no less privacy and accuracy may exist. For algorithms beyond noise-addition, \citet{bhaskara2012unconditional} shows that they can be reduced to a noise-addition mechanism with better accuracy and slightly worse privacy.

Another exciting question is to search for distributions such that \eqref{eq:newCR} is tight, where \citet{dagan2020bounded,ghazi2020avoiding} can serve as good hints. Since \citet{dagan2020bounded} uses bounded noise, a CLT compatible with bounded log-concave noise is also desirable.

\section*{Acknowledgements}
We thank Jason Hartline, Yin-tat Lee, Haotian Jiang, Qiyang Han, Sasho Nikolov, and Yuansi Chen for helpful comments on earlier versions of the manuscript. W.~J.~S.~was supported in part by NSF through CCF-1763314 and CAREER DMS-1847415, an Alfred Sloan Research Fellowship, and a Facebook Faculty Research Award. L.~Z.~was supported in part by NSF through DMS-2015378.

\bibliography{privacy}
\bibliographystyle{alpha}


\begin{appendices}
\section*{Appendices Overview} 
\label{sec:overview}

In Appendix \ref{sub:details_of_empirical_trade_off_functions} we provide the detail of the numerical experiments. The central limit theorem \ref{thm:main} (normalized) and \ref{cor:main} (general case) are proved in \Cref{sub:proofs_in_sec:clt}. \Cref{sub:proofs_in_sub:uncertainty} proves \Cref{thm:variance} and provide some additional results that apply beyond norm powers. The proof of \Cref{lem:normpower}, which verifies that norm powers satisfy the technical conditions \ref{d1} and \ref{d2}, requires results in \Cref{sub:proofs_in_sub:uncertainty} and takes significant effort, so we dedicate the entire \Cref{sec:proof-case-x_palpha} to it.
\section{Numerical Verification of the Central Limit Theorem} 
\label{sub:details_of_empirical_trade_off_functions}
This section discusses the details of the numerical experiments shown in \Cref{fig:N} (repeated below) that verifies our central limit theorem.
\begin{figure}[!h]
    \centering
    \includegraphics[width=\textwidth]{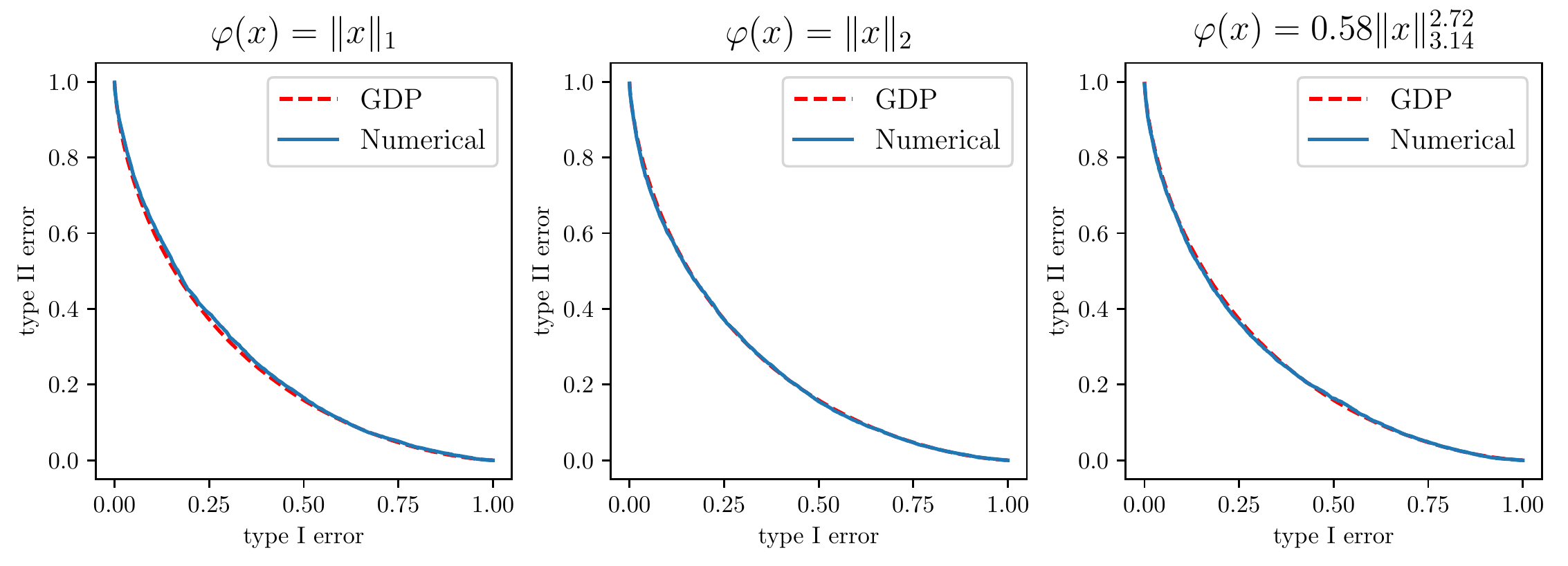}
    \vskip-5pt
    \caption*{\Cref{fig:N}: Fast convergence to GDP as claimed in \Cref{thm:informal}. Blue solid curves indicate the true privacy (i.e. ROC functions, see \Cref{sec:preliminary} for details) of the noise addition mechanism considered in \Cref{thm:informal}. Red dashed curves are GDP limit predicted by our CLT. In all three panels the dimension $n=30$. In order to show that our theory works for general $\ell_p$-norm to the power $\alpha$, we pick them to be famous mathematical constants, namely $p=\pi,\alpha=\e$ and the coeffcient being the Euler–Mascheroni constant $\gamma$.}
\end{figure}
The mechanism in consideration is $M(D)=f(D)+tX$ where the $n$-dimensional random vector $X$ has density $\propto \e^{-\|x\|_p^\alpha}$. We want to demonstrate that when $t={\mu^{-1}}\cdot \sqrt{\|\Fisher_{p,\alpha}\|_2}={\mu^{-1}}\cdot n^{\frac{1}{p}-\frac{1}{\alpha}}\cdot\alpha^{\frac{1}{\alpha}}\cdot p^{1-\frac{1}{p}}\cdot\sqrt{\frac{\Gamma(2-\frac{1}{p})}{\Gamma(\frac{1}{p})}}\cdot 1+(o(1))$ we have
    \[\inf_{D,D'} T[M(D),M(D')]\approx G_\mu.\]
The infimum is taken over $D,D'$ such that the direction of $f_n(D)-f_n(D')$ is in a large subset $E_n$ of the unit sphere. It is hard to evaluate the infimum even numerically, but it turns out that the infimum is equal to $\inf_{v\in E_n} T[tX,tX+v]$. This is the first part of the proof of \Cref{thm:informal}. 

Therefore, it suffices to evaluate $T(tX,tX+v)$ and compare with the GDP function $G_\mu$, but the high-dimensional nature of $X$ prevents exact evaluation, so we will introduce a Monte Carlo approach.

It suffices to evaluate $T(tX,tX+v)$ and compare with the GDP function $G_\mu$, but the high-dimensional nature of $X$ prevents exact evaluation, so we will introduce a Monte Carlo approach.

\paragraph{Empirical ROC Function} 
\label{par:empirical_trade_off_function}
$X$ has density $\propto\e^{-\varphi(x)}$ and $X+v$ has density $\propto\e^{-\varphi(x-v)}$. The log likelihood ratio is $\varphi(x)-\varphi(x-v)$. Thresholding it at $h$ yields the following type I and type II errors
\begin{align*}
    \alpha(h) &= \P_{x\sim X}(\varphi(x)-\varphi(x-v)\geqslant h)=\P(\varphi(X)-\varphi(X-v)\geqslant h)\\
    \beta(h) &= \P_{x\sim X+v}(\varphi(x)-\varphi(x-v)< h)=\P(\varphi(X+v)-\varphi(X)< h)
\end{align*}
Once we have these, $T(X,X+v)$ can be obtained by eliminating $h$ and express $\beta$ as a function of $\alpha$. These two probabilities can be computed by a simple Monte Carlo approach. First 
We can sample $\{x_1,\ldots, x_N\}$ as i.i.d. copies of $X$. Let
\begin{align*}
    a_i &= \varphi(x_i-v)-\varphi(x_i)\\
    b_i &= \varphi(x_i+v)-\varphi(x_i)\\
    \hat{\alpha}(h) &= \frac{1}{N} \cdot\#\{a_i\leqslant-h\}\\
    \hat{\beta}(h) &= \frac{1}{N}\cdot \#\{b_i< h\}
\end{align*}
We only evaluate on a discrete set of $h$ such that the corresponding $\alpha$ forms a uniform grid $\{\frac{1}{N},\ldots, \frac{N}{N}\}$. Let $h_j = -a_{(j)}$ where $a_{(1)}\leqslant\cdots \leqslant a_{(N)}$ are order statistics of $a_i,i=1,2,\ldots, N$. Then for $j=0,1,2,\ldots, N$,
\begin{align*}
     \hat{\alpha}_j &= \hat{\alpha}(h_j)=\frac{1}{N}\cdot \#\{a_i\leqslant a_{(j)}\}=\frac{j}{N}\\
     \hat{\beta}_j &= \hat{\beta}(h_j) = \frac{1}{N}\cdot \#\{b_i< a_{(j)}\}
\end{align*}
Let $\hat{T}_N(X,X+v)$ be the function that linearly interpolates the values $\hat{\beta}_0,\ldots, \hat{\beta}_N$ at $\frac{0}{N},\ldots, \frac{N}{N}$. As a direct consequence of the well-known Glivenko--Cantelli theorem, we have
\[\|\hat{T}_N(X,X+v)-{T}(X,X+v)\|_\infty
\to 0 \text{ almost surely.}\] 

\paragraph{Evaluating the Fisher Information} 
Note that it is numerically infeasible to use the exact expression in \Cref{thm:variance}, since Gamma function grows extremely fast (of course it does, as an interpolation of the factorial). In practice, we find the asymptotic expression in \Cref{thm:variance} works extremely well.

\begin{figure}[!h]
    \centering
    \includegraphics[width=\textwidth]{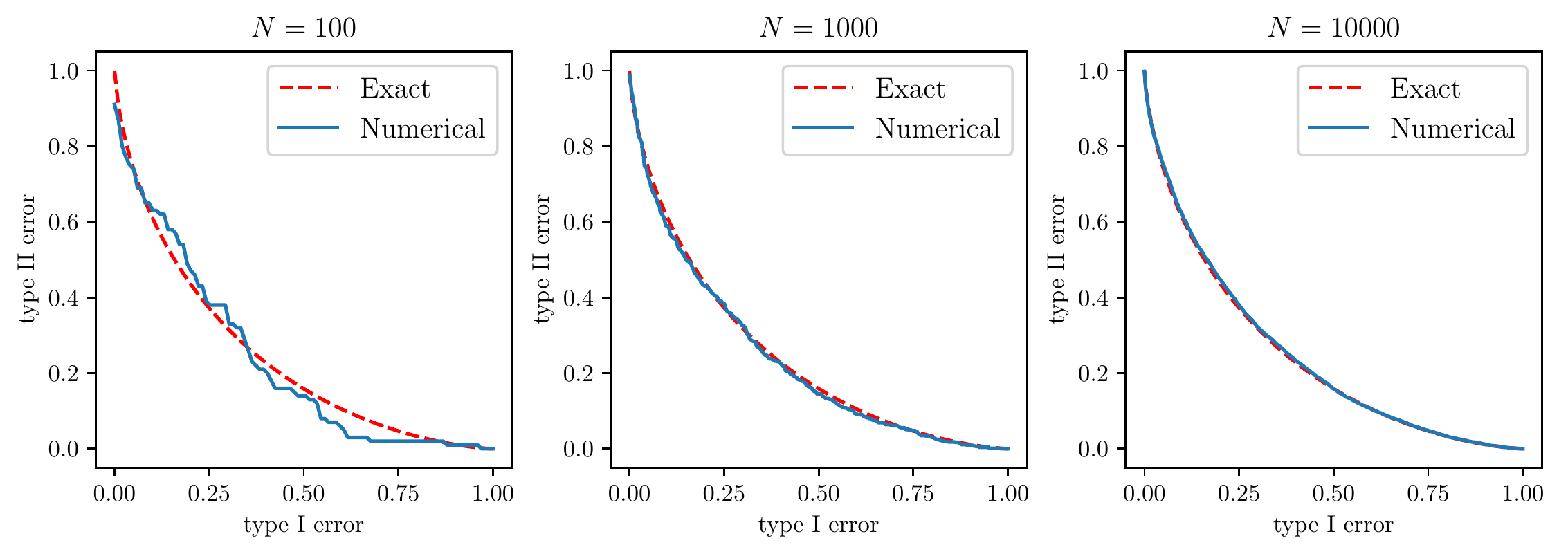}
    \vskip-5pt\caption{The effect of sample size $N$. We see that $N=10000$ provides a good estimate.
    The red dashed curve corresponds to Gaussian noise $\varphi(x)=\frac{1}{2}\|x\|_2^2$, which is why we can perform exact computation. The dimension $n=30$.}
    \label{fig:sample}
\end{figure}

Next we present the algorithm that samples an $n$-dimensional random vector whose density is $\propto\e^{-\|x\|_p^\alpha}$.
\begin{algorithm}[!htb]
    \caption{\texttt{Sample $\propto\e^{-\|x\|_p^\alpha}$}}\label{algo:sample}
    \begin{algorithmic}[1]
    \State{\bf Input:} $p,\alpha$ and dimension $n$
        \Statex Generate $t\sim \Gamma(\frac{n}{\alpha}+1,1)$
        \Statex Generate $\xi_i,i=1,2,\ldots, n$ i.i.d. from $\Gamma(\frac{1}{p},1)$
        \Statex Generate random vector $\bm x\in\R^n$ where $x_i=\epsilon_i\cdot\xi_i^{\frac{1}{p}}$, $\epsilon_i$ are Rademacher random variables (unbiased coin flips) independent from everything else.
        \Statex Generate $r\sim U[0,1]$
        \Statex Let $V=r^{\frac{1}{n}}\cdot \frac{x}{\|x\|_p}$
        \State {\bf Output} $t\cdot V$
    \end{algorithmic}
\end{algorithm}

It is easy to see its correctness from \Cref{thm:uniform} and \cite{calafiore1998uniform}.

\newif\ifhideproofs

\ifhideproofs
\let\proof\hide
\let\endproof\endhide
\fi

\section{Proof of \Cref{thm:main} and \ref{cor:main}} 
\label{sub:proofs_in_sec:clt}

In this section we first prove the normalized central limit theorem \ref{thm:main} and then the general case theorem \ref{cor:main}. Recall that in normalized CLT,
the mechanism in consideration is $M_n(D)=f_n(D)+ X_{n}$, where $X_n$ has density $\propto\e^{-\varphi_n}$ with Fisher information $\Fisher_n$ being the identity matrix $I_{n\times n}$.
\main*
\begin{enumerate}[label=(D\arabic*)]
    \item$\KS(P_v^{\varphi_n}(X_n),v^T\nabla\varphi_n(X_n)=o(1)$ with probability at least $1-o(1)$ over $v\sim S^{n-1}$
	\item$\|\nabla\varphi_n(X_n)\|_2=\sqrt{n}\cdot(1+o_{P}(1))$
\end{enumerate}

\begin{proof}[Proof of \Cref{thm:main}]
For clarity, in the proof we drop the subscript $n$ unless the limit $n\to\infty$ is taken.
First we show that 
\[\inf_{D,D'} T[M(D),M(D')]=\inf_{v\in E_n} T[X,X+v].\]
Notice that
\begin{align*}
	T[M(D),M(D')]
	=T[f(D)+X,f(D')+X]&=T[X,X+f(D')-f(D)]
\end{align*}
Consider the vector $f(D')-f(D)$. Let $v=\tfrac{f(D')-f(D)}{\|f(D')-f(D)\|_2}$ be its direction and $r$ be its length. We have $f(D')-f(D)=rv$. Since $f$ has $\ell_2$-sensitivity 1, we have $r\in[0,1]$.
The infimum over $D,D'$ can be taken in two steps: first over $v$ and then over $r$. That is,
\begin{align*}
\inf_{D,D'} T[M(D),M(D')] = \inf_{v\in E_n} \inf_{r\in[0,1]} T[M(D),M(D')]
= \inf_{v\in E_n} \inf_{r\in[0,1]} T[X,X+rv]
\end{align*}
By \Cref{lem:monotone}, $T[X,X+rv]$ is pointwise monotone decreasing in $r$, so for the inner infimum we have $\inf_{r\in[0,1]} T[X,X+rv] = T[X,X+v]$.
Back to the limiting conclusion, it suffices to show that for all $v\in E_n$,
\[\|T[X_n,X_n+v]-G_1\|_\infty\leqslant c_n.\]
To prove this, we use the following lemma
\begin{lemma}\label{lem:ROC_compute}
Suppose random vector $X$ has density $\propto \e^{-\varphi}$ where $\varphi\in\mathfrak{F}_n$. Let $F_v$ be the CDF of the likelihood projection
$P_v^{\varphi}(X)=\varphi(X+v)-\varphi(X)-\frac{1}{2}v^T\Fisher_\varphi v$, then for any $v\in\R^n$,
	\[
		T[X,X+v](\alpha)= F_v\big(-F_{-v}^{-1}(\alpha)-v^T\Fisher_\varphi v\big).
	\]
\end{lemma}
Let $H_v^n$ and $F_v^n$ be the CDFs of the linear projection $v^T\nabla\varphi_n(X_n)$ and the likelihood projection $P_v^{\varphi_n}(X_n)=\varphi_n(X_n+v)-\varphi_n(X_n)-\frac{1}{2}v^T\Fisher_n v$. When \Cref{lem:ROC_compute} is applied to $X_n$ and unit vector $v$, we have
\begin{equation}\label{eq:g1}
	T[X_n,X_n+v](\alpha) = F_v^n\big(-(F_{-v}^n)^{-1}(\alpha)-1\big).
\end{equation}
Recall that $G_1(\alpha)=\Phi\big(-\Phi^{-1}(\alpha)-1\big)$ where $\Phi$ is the CDF of standard normal.
Comparing with \eqref{eq:g1}, it suffices to show $F_v^n$ is close to $\Phi$.
To prove this and to take care of the set $E_n\subseteq S^{n-1}$, we use the two conditions \ref{d1} and \ref{d2}, which allow us to apply Sudakov's Theorem (c.f. \cite{klartag2010high}) stated below as \Cref{thm:Sudakov}. In the following, let $\sigma_n$ be the uniform measure (with total measure 1) on the unit sphere $S^{n-1}$.
\begin{lemma}\label{thm:Sudakov}
	Let $Y_n$ be an isotropic random vector in $\R^n$. Assume that there is $a_n\to0$ and
	\begin{equation}\label{eq:thin-shell}
		\P\left(\left|\frac{\|Y_n\|_2}{\sqrt{n}}-1\right|\geqslant a_n\right)\leqslant a_n.
	\end{equation}
Then, there exists $b_n\to 0$ and $\Theta_n\subseteq S^{n-1}$ with $\sigma_n( \Theta_n)\geqslant 1-b_n$, such that
for any $v\in\Theta_n$,
\[
\sup_{t\in\R}|\P(v^TY_n\leqslant t)-\Phi(t)|\leqslant b_n.
\]
\end{lemma}

Let $Y_n = \varphi_n(X_n)$. $v^T\nabla\varphi_n(X_n)=v^TY_n$ We know $\E Y_n = \E\nabla\varphi_n(X_n)=0$, and by the normalization of the Fisher information, 
\[\E Y_nY_n^T = \Fisher_n = I_{n\times n}.\]
Therefore, $Y_n$ is isotropic.
	Condition \ref{d2} says $\|Y_n\|_2 = \sqrt{n}\cdot(1+o_{P}(1))$. That is, 
		$\left|\frac{\|Y_n\|_2}{\sqrt{n}}-1\right|=o_{P}(1)$. This implies the existence of $a_n$ in \eqref{eq:thin-shell}. Therefore, we can apply \Cref{thm:Sudakov} to $Y_n=\varphi_n(X_n)$ and conclude that there is $b_n\to 0$ and $\Theta_n\subseteq S^{n-1}$ with $\sigma_n( \Theta_n)\geqslant 1-b_n$, such that
for any $v\in\Theta_n$,
\[
\|H_v^n-\Phi\|_\infty\leqslant b_n.
\]

Condition \ref{d1} says, there is $b_n'\to0$ and $\Omega_n\subseteq S^{n-1}\geqslant b_n'$ with $\sigma_n( \Omega_n)$ such that for all $v\in \Omega_n$,
	\begin{align*}
		\|F_v^n-H_v^n\|_\infty\leqslant b_n'
	\end{align*}

So $v\in \Theta_n\cap\Omega_n$ implies $\|H_v^n-\Phi\|_\infty\leqslant b_n$ and $\|F_v^n-H_v^n\|_\infty\leqslant b_n'$. Therefore, $\|F_v^n-\Phi\|_\infty\leqslant b_n+b_n'$.
Set $E_n=\Theta_n\cap (-\Theta_n)\cap\Omega_n\cap (-\Omega_n)$.
 Then $v\in E_n$ implies
 \[
 \|F_v^n-\Phi\|_\infty\leqslant b_n+b_n'\text{ and } \|F_{-v}^n-\Phi\|_\infty\leqslant b_n+b_n'
 \]
That is, for any $x\in\R$,
\begin{align*}
	\Phi(x)-b_n-b_n'&\,\leqslant F_{ v}^{n}(x)\,\,\leqslant\Phi(x)+b_n+b_n'\\
	\Phi(x)-b_n-b_n'&\leqslant F_{-v}^{n}(x)\leqslant\Phi(x)+b_n+b_n'
\end{align*}
As a consequence of the second inequality,
\[
	\Phi^{-1}(\alpha-b_n-b_n')\leqslant (F_{-v}^{n})^{-1}(\alpha)\leqslant\Phi^{-1}(\alpha+b_n+b_n')
\]
Therefore, when $v\in\E_n$, by \eqref{eq:g1} and the inequalities above,
	\begin{align*}
	T[X_n,X_n+v](\alpha)
	&= F_v^n\big(-(F_{-v}^n)^{-1}(\alpha)-1\big)\\
	&\leqslant F_v^n\big(-\Phi^{-1}(\alpha-b_n-b_n')-1\big)\\
	&\leqslant \Phi\big(-\Phi^{-1}(\alpha-b_n-b_n')-1\big)+b_n+b_n'\\
	&= G_1(\alpha-b_n-b_n')+b_n+b_n'\\
	&\leqslant G_1(\alpha)+C\sqrt{b_n+b_n'}+b_n+b_n'
	\end{align*}
The final step used the Hölder continuity of $G_\mu$ which we state as \Cref{lem:holder} and prove afterwards.
\begin{lemma}\label{lem:holder}
	Let $G_\mu=T[N(0,1),N(\mu,1)]$ for $\mu\geqslant0$. Then $G_\mu$ is $\alpha$-Hölder continuous for any $\alpha<1$.
\end{lemma}
It is worth noting that $G_\mu$ is not $1$-Hölder continuous (i.e. Lipschitz continuous) as long as $\mu>0$.

Without loss of generality assume $C>1$. Let $c_n = C\sqrt{b_n+b_n'}+b_n+b_n'$. Then $c_n\geqslant 2b_n+2b_n'$. The above argument shows $T[X_n,X_n+v](\alpha)\leqslant G_1(\alpha) + c_n$. The lower bound can be obtained similarly, so for $v\in E_n$, we have
\[\|T[X_n,X_n+v]-G_1\|_\infty\leqslant c_n.\]
Since all four sets $\Theta_n, -\Theta_n, \Omega_n$ and $-\Omega_n$ are large, we have
\[\sigma_n(E_n) = \sigma_n(\Theta_n\cap (-\Theta_n)\cap\Omega_n\cap (-\Omega_n))
\geqslant 1-(2b_n+2b_n')\geqslant 1-c_n
\]
This is the conclusion stated in the theorem. The proof is complete.
\end{proof}
Next we provide the proofs of the lemmas used, namely \Cref{lem:monotone,lem:holder,lem:ROC_compute}.

Let $M^t(D)=f(D)+tX$.
\monotone*
\begin{proof}[Proof of \Cref{lem:monotone}]
	Let $0\leqslant t_1\leqslant t_2$ and $f_i=T[X,X+t_iv], i=1,2$. We need to show $f_1\geqslant f_2$. Fix $\alpha\in[0,1]$, let $A_{t}\subseteq\R^n$ be the optimal rejection region for the testing of $X$ vs $X+tv$. That is,
	$$
		\P[X\in A_t]=\alpha\quad\text{and}\quad \P[X+tv\not\in A_t]=T[X,X+tv](\alpha)
	$$
	In order to show $f_1(\alpha)\geqslant f_2(\alpha)$, 
	consider a translated set $A_{t_1}+(t_2-t_1)v$. This set is at the best suboptimal for the testing of $X$ vs $X+t_2v$. If we denote $\P[X\in A_{t_1}+(t_2-t_1)v]$ by $\alpha'$, suboptimality means
	\[\P[X+t_2v\not\in A_{t_1}+(t_2-t_1)v]\geqslant f_2(\alpha').\]
	If we can show $\alpha'\leqslant \alpha$, then by the monotonicity of ROC functions, we have
	\begin{align*}
		f_2(\alpha)&\leqslant f_2(\alpha')\\
		&=\P[X+t_2v\not \in A_{t_1}+(t_2-t_1)v]\\
		&=\P[X+t_1v\not \in A_{t_1}]\\
		&=f_1(\alpha)
	\end{align*}
	So the only thing left is to show $\alpha'\leqslant \alpha$, or equivalently,
	\[\P[X\in A_{t_1}+(t_2-t_1)v]\leqslant\P[X\in A_{t_1}].\]
	In fact, we will show $A_{t_1}+(t_2-t_1)v\subseteq A_{t_1}$. This is where log-concavity kicks in. To phrase it more generally, we are going to show that $A_{t}+sv\subseteq A_{t}$ for general $t,s\geqslant 0$. Suppose $X$ has density $\e^{-\varphi(x)}$ where $\varphi:\R^n\to{\R\cup\{+\infty\}}$ is a (potentially extended) convex function. By Neyman--Pearson lemma, $A_{t}=\{x:\varphi(x)-\varphi(x-tv)>h\}$ for some threshold $h$. We would like to show that $x\in A_t$ implies $x+sv\in A_t$. It suffices to show
	\[
		\varphi(x+sv)-\varphi(x+sv-tv)\geqslant\varphi(x)-\varphi(x-tv).
	\]
	In fact, $\varphi(x+sv)-\varphi(x+sv-tv)$ is monotone increasing as a function of $s$. This is a direct consequence of the convexity of $\varphi(x+sv)$ as a function of $s$. More specifically, let $g(s)=\varphi(x+sv)$. Its convexity follows from the convexity of $\varphi$. For $t\geqslant0$, $\varphi(x+sv)-\varphi(x+sv-tv)=g(s)-g(s-t)$ is easily seen to be monotone by taking a derivative, or from a more rigorous approach following just the definition of convex functions.
\end{proof}

Recall that $\tilde{\varphi}_n(x)=\varphi_n(\frac{x}{t_n})$. The corresponding random vector with density $\propto \e^{-\tilde{\varphi}_n}$ is $\tilde{X}_n = X_{\tilde{\varphi}_n}$ and has the same distribution as $t_n X_{\varphi_n}$. The scaling factor normalizes its Fisher information to be an $n\times n$ scalar matrix. In fact,
$$\tilde{\Fisher}_n:=\Fisher_{\tilde{\varphi}_n}=t_n^{-2}\Fisher_{{\varphi}_n}=\mu^2 \cdot I_{n\times n}.$$
\begin{proof}[Proof of \Cref{lem:holder}]
	We know that $G_\mu(x)=\Phi\big(-\Phi^{-1}(x)-\mu\big)$. It suffices to show that $1-G_\mu(x)=\Phi\big(\Phi^{-1}(x)+\mu\big)$ is $\alpha$-Hölder continuous for any $\alpha<1$.
	\begin{lemma}
		Consider $f,g:[0,1]\to\R$. Suppose $f,g$ and $g-f$ is monotone increasing and $g$ is $\alpha$-Hölder continuous, then $f$ is also $\alpha$-Hölder continuous.
	\end{lemma}
	To see this, notice that by monotonicity of $g-f$, we have $g(x)-f(x)\leqslant g(y)-f(y)$ for $x<y$. Hence
	\begin{align*}
		|f(y)-f(x)|=f(y)-f(x)\leqslant g(y)-g(x)\leqslant Cx^\alpha.
	\end{align*}
	\begin{lemma}
		For each $\alpha<1$, there is an $\ep=\ep(\alpha)>0$ such that $x^{\alpha}-(1-G_\mu(x))$ is monotone increasing in $[0,\ep]$.
	\end{lemma}
	Let $\alpha=1-\delta$. $h(x)=x^{\alpha}-(1-G_\mu(x))=x^{\alpha}-\Phi\big(\Phi^{-1}(x)+\mu\big)$. Let $y=\Phi^{-1}(x)$. Then $x=\Phi(y)$
	\begin{align*}
		h'(x)
		&=\alpha x^{\alpha-1}-\frac{\diff}{\diff x}\Phi(y+\mu)\\
		&=\alpha x^{-\delta}-\phi(y+\mu)\cdot \frac{\diff y}{\diff x}\\
		&=\alpha x^{-\delta}-\phi(y+\mu)\big\slash\frac{\diff x}{\diff y}\\
		&=\alpha x^{-\delta}-\frac{\phi(y+\mu)}{\phi(y)}\\
		&=\alpha x^{-\delta}-\e^{-\mu y-\frac{1}{2}\mu^2}\\
		&=\e^{\delta\log x^{-1}+\log\alpha}-\e^{-\mu y-\frac{1}{2}\mu^2}
	\end{align*}
	It is known that $|\Phi^{-1}(x)|=-\Phi^{-1}(x)\leqslant \sqrt{2\log x^{-1}}$.
	So for fixed $\alpha,\delta$ and $\mu$, there is an $\ep$ such that when $x\in[0,\ep]$, we have
	\begin{align*}
	\delta\log x^{-1}+\log\alpha\geqslant{\mu \Phi^{-1}(x)-\frac{1}{2}\mu^2}={-\mu y-\frac{1}{2}\mu^2}
	\end{align*}
	Hence,
	\begin{align*}
		h'(x)
		&=\e^{\delta\log x^{-1}+\log\alpha}-\e^{-\mu y-\frac{1}{2}\mu^2}\geqslant0
	\end{align*}
\end{proof}
Interestingly, this implies the following result:
\begin{proposition}
	For each $\alpha\in[0,1)$, there is a $C>0$ such that
	\[\int_{a+1}^{b+1}\e^{-x^2}\diff x\leqslant C\left(\int_{a}^{b}\e^{-x^2}\diff x\right)^\alpha.\]
\end{proposition}
For a convex $\varphi$ such that $\e^{-\varphi}$ is integrable, let $F_v^\varphi$ be the cdf of $P_v^\varphi(X_\varphi)$. Dropping the unnecessary subscripts and superscripts of $\varphi$, we have
\begin{proof}[Proof of \Cref{lem:ROC_compute}]
We are interested in the hypothesis testing $H_0: X \text{ vs } H_1: X+v$. By definition of ROC function in \Cref{def:ROC}, we need to find out the optimal type II error at a given level $\alpha$. By Neymann--Pearson lemma, it suffices to consider likelihood ratio tests. The log density of the null is (up to an additive constant) $-\varphi(x)$, while that of the alternative is $-\varphi(x-v)$. 
	So the log likelihood ratio is $\varphi(x)-\varphi(x-v)$. Under null it is distributed as $\varphi(X)-\varphi(X-v)$ and thresholding at $h$ yields type I error
	\begin{align*}
		\alpha
		&= \P(\varphi(X)-\varphi(X-v)>h)\\
		&= \P(\varphi(X-v)-\varphi(X)<-h)\\
		&= \P(\varphi(X-v)-\varphi(X)-\tfrac{1}{2}v^T\Fisher_\varphi v<-h-\tfrac{1}{2}v^T\Fisher_\varphi v)\\
		&=F_{-v}(-h-\tfrac{1}{2}v^T\Fisher_\varphi v)
	\end{align*}
	Under the alternative, the log likelihood ratio is distributed as $\varphi(X+v)-\varphi(X)$, so the corresponding type II error is
	\begin{align*}
		\beta
		&= \P(\varphi(X+v)-\varphi(X)<h)\\
		&= \P(\varphi(X+v)-\varphi(X)-\tfrac{1}{2}v^T\Fisher_\varphi v<hh-\tfrac{1}{2}v^T\Fisher_\varphi v)\\
		&=F_{v}(h-\tfrac{1}{2}v^T\Fisher_\varphi v)
	\end{align*}
	From the expression of $\alpha$ we can solve for $h$:
	\[h=-F_{-v}^{-1}(\alpha)-\tfrac{1}{2}v^T\Fisher v\]
	Plugging this into the expression of $\beta$ yields
	\begin{align*}
		\beta
		&=F_{v}(h-\tfrac{1}{2}v^T\Fisher_\varphi v)\\
		&=F_v\big(-F_{-v}^{-1}(\alpha)-v^T\Fisher v\big)
	\end{align*}
	The ROC function maps $\alpha$ to the minimal $\beta$, so this is exactly the expression of $T[X,X+v]$.
\end{proof}

	


\section{Proof of \Cref{thm:variance}} 
\label{sub:proofs_in_sub:uncertainty}

The major goal of this section is the following extended version of \Cref{thm:variance}. We proceed by first presenting some general results in \Cref{sub:regarding_homogeneous_}, followed by calculation for norm powers in \Cref{sub:calculation_for_norm_powers}.


\begin{lemma}
\label{lem:variance}
	For $1\leqslant p<\infty$ and $1\leqslant \alpha<\infty$, as $n\to\infty$, we have
\begin{align*}
	\E\|X_{p,\alpha}\|_2^2
	&=\frac{\Gamma(\frac{n}{\alpha}+1+\frac{2}{\alpha})}{\Gamma(\frac{n}{p}+1+\frac{2}{p})}\cdot 	\frac{\Gamma(\frac{n}{p}+1)}{\Gamma(\frac{n}{\alpha}+1)}\cdot 	n\cdot\frac{\Gamma(\tfrac{3}{p})}{\Gamma(\frac{1}{p})}\\
	&\sim n^{\frac{2}{\alpha}-\frac{2}{p}+1} \cdot \alpha^{-\frac{2}{\alpha}}\cdot p^{\frac{2}{p}}\cdot\frac{\Gamma(\tfrac{3}{p})}{\Gamma(\frac{1}{p})}\\
	\Fisher_{p,\alpha}=&\alpha^2\cdot \frac{\Gamma(\frac{n+2\alpha-2}{\alpha})}{\Gamma(\frac{n+2p-2}{p})}\cdot
	\frac{\Gamma(\frac{n}{p})}{\Gamma(\frac{n}{\alpha})}\cdot
	\frac{\Gamma(2-\frac{1}{p})}{\Gamma(\frac{1}{p})}\cdot I_{n\times n}\\
	&\sim n^{\frac{2}{p}-\frac{2}{\alpha}}\cdot\alpha^{\frac{2}{\alpha}}\cdot p^{2-\frac{2}{p}}\cdot\frac{\Gamma(2-\frac{1}{p})}{\Gamma(\frac{1}{p})} \cdot  I_{n\times n}.
\end{align*}
\end{lemma}

\subsection{Regarding Homogeneous $\varphi$} 
\label{sub:regarding_homogeneous_}

In this section, in addition to that $\varphi\in\mathfrak{F}_n$, we further assume that $\varphi:\R^n\to\R$ is positively homogeneous. Recall that $\varphi$ is (positively) homogeneous of degree $\alpha>0$ if $\varphi(tx) = |t|^\alpha \varphi(x)$ for $t\in\R,x\in\R^n$. This implies $\varphi(0)=0$.

The first result takes care of the normalizer $Z_\varphi$ defined as $\int \e^{-\varphi(x)}\diff x$.
\begin{lemma}\label{lem:partition}
	$Z_\varphi=\Gamma(\tfrac{n}{\alpha}+1) \cdot \vol(K_\varphi)$ where $K_\varphi=\{x:\varphi(x)\leqslant 1\}$.
\end{lemma}
\begin{proof}[Proof of \Cref{lem:partition}]
We use polar coordinate. For any function $f:\R^n\to \R$,
\begin{align*}
	\int_{\R^n} f(x) \diff x 
	= \int_0^\infty \int_{S^{n-1}} f(r\theta)r^{n-1} \diff \theta\diff r
\end{align*}
So
\begin{align*}
	Z=\int_{\R^n} \e^{-\varphi(x)} \diff x
	&= \int_{S^{n-1}} \int_0^\infty  \e^{-\varphi(r\theta)}r^{n-1} \diff r\diff \theta\\
	&= \int_{S^{n-1}} \int_0^\infty  \e^{-r^\alpha\varphi(\theta)}r^{n-1} \diff r\diff \theta
\end{align*}
Let $t=r^\alpha, r=t^{\frac{1}{\alpha}}, \diff r=\frac{1}{\alpha}t^{\frac{1}{\alpha}-1}\diff t$.
\begin{align*}
	\int_0^\infty  \e^{-r^\alpha\varphi(\theta)}r^{n-1} \diff r 
	&= \int_0^\infty  \e^{-t\varphi(\theta)}\cdot t^{\frac{n-1}{\alpha}}\cdot\tfrac{1}{\alpha}t^{\frac{1}{\alpha}-1} \diff t\\
	&= \frac{1}{\alpha}\int_0^\infty  \e^{-t\varphi(\theta)}\cdot t^{\frac{n}{\alpha}-1} \diff t\\
	&= \frac{1}{\alpha} \cdot \frac{\Gamma(\frac{n}{\alpha})}{\varphi(\theta)^{\frac{n}{\alpha}}}
\end{align*}
So
\begin{align*}
	Z = \int_{S^{n-1}} \frac{1}{\alpha} \cdot \frac{\Gamma(\frac{n}{\alpha})}{\varphi(\theta)^{\frac{n}{\alpha}}} \diff \theta
\end{align*}
On the other hand, consider a set defined with polar coordinate:
\[K:=\{(r,\theta): r\leqslant \rho(\theta)\}.\]
Its volume is
\begin{align*}
	\vol(K) &= \int_{\R^n} 1_K(x) \diff x \\
	&= \int_{S^{n-1}} \int_0^{\rho(\theta)}  r^{n-1} \diff r\diff \theta\\
	&= \frac{1}{n} \int_{S^{n-1}} \rho(\theta)^{n}\diff \theta
\end{align*}
We see that
\begin{align*}
	Z
	&= \int_{S^{n-1}} \frac{1}{\alpha} \cdot \frac{\Gamma(\frac{n}{\alpha})}{\varphi(\theta)^{\frac{n}{\alpha}}} \diff \theta\\
	&= \frac{\Gamma(\frac{n}{\alpha})}{\alpha} \cdot n\cdot \frac{1}{n} \int_{S^{n-1}} \big[\varphi(\theta)^{-\frac{1}{\alpha}}\big]^n \diff \theta\\
	&= \tfrac{n}{\alpha}\cdot\Gamma(\tfrac{n}{\alpha}) \cdot \vol(K_\varphi)
\end{align*}
where
\begin{align*}
	K_\varphi=\{(r,\theta): r\leqslant \varphi(\theta)^{-\frac{1}{\alpha}}\}
	=\{(r,\theta): r^\alpha\varphi(\theta)\leqslant 1\}
	=\{(r,\theta): \varphi(r\theta)\leqslant 1\}
	=\{x:\varphi(x)\leqslant 1\}.
\end{align*}
Noticing $\Gamma(z+1)=z\Gamma(z)$, we have
\[Z=\Gamma(\tfrac{n}{\alpha}+1) \cdot \vol(K_\varphi)\]
\end{proof}

\begin{lemma} \label{lem:gamma_moment}
	The $m$-th moment of a $\Gamma(k,1)$ distribution is 
	$\frac{\Gamma(m+k)}{\Gamma(k)}$.
\end{lemma}
\begin{proof}[Proof of \Cref{lem:gamma_moment}]
The $m$-th moment of a $\Gamma(k,1)$ distribution is 
	$$\frac{1}{\Gamma(k)}\int_0^\infty x^m\cdot x^{k-1}\e^{-x}\diff x = \frac{\Gamma(m+k)}{\Gamma(k)}.$$
\end{proof}

The following result also appears in \cite{wang2005volumes}.
\begin{lemma} \label{lem:vol}
	\[\vol(K_p) = 2^n\cdot\frac{\Gamma(\frac{1}{p}+1)^n}{\Gamma(\frac{n}{p}+1)}.\]
\end{lemma}
\begin{proof}[Proof of \Cref{lem:vol}]
	By \Cref{lem:partition}, we have
	\begin{align*}
		\vol(K_p)
		&= \frac{1}{\Gamma(\frac{n}{p}+1)}\cdot \int\e^{-\sum|x_i|^p}\diff x
		\\&=\frac{1}{\Gamma(\frac{n}{p}+1)}\cdot \left(\int_{-\infty}^{+\infty}\e^{-|x|^p}\diff x\right)^n
	\end{align*}
	\begin{align*}
		\int_{-\infty}^{+\infty}\e^{-|x|^p}\diff x
		&=2\int_{0}^{+\infty}\e^{-x^p}\diff x\\
		&=2\int_{0}^{+\infty}\e^{-y}\diff y^{\frac{1}{p}}\\
		&=\frac{2}{p}\int_{0}^{+\infty}\e^{-y} y^{\frac{1}{p}-1}\diff y\\
		&=\tfrac{2}{p}\Gamma(\tfrac{1}{p})=2\Gamma(\tfrac{1}{p}+1)
	\end{align*}
\end{proof}

\begin{lemma} \label{thm:uniform}
	Let $t\sim \Gamma(\frac{n}{\alpha}+1,1)$. Let $V_\varphi$ has uniform distribution over $K_\varphi$ where $K_\varphi=\{x:\varphi(x)\leqslant 1\}$ independently from $t$. Then $t^{\frac{1}{\alpha}}\cdot V$ has density $\frac{1}{Z}\e^{-\varphi}$.
\end{lemma}

\begin{proof}[Proof of \Cref{thm:uniform}]
	We use a more principled way: assume $r$ has density $p(r)$ over $(0,+\infty)$ and $rV$ has density $\frac{1}{Z}e^{-\varphi}$, find $p(r)$. Let $B$ be a small ball.
	\begin{align*}
		\P(rV\in x+B) = \int_0^\infty \P(V\in\frac{x+B}{r})\cdot p(r)\diff r
	\end{align*}
	\[
		\P(V\in\frac{x+B}{r})=
		\begin{dcases*}
				0, & if $\frac{x}{r}\notin K_\varphi$\\
				\frac{\vol(\frac{B}{r})}{\vol(K_\varphi)}, &if $\frac{x}{r}\in K_\varphi$.
		\end{dcases*}
	\]
	$\frac{x}{r}\in K_\varphi \Leftrightarrow \varphi(\frac{x}{r})\leqslant 1 \Leftrightarrow \varphi(x)\leqslant r^\alpha \Leftrightarrow  r\geqslant \varphi(x)^{1/\alpha}$.
	So
	\begin{align*}
		\P(rV\in x+B) &= \int_0^\infty \P(V\in\frac{x+B}{r})\cdot p(r)\diff r\\
		&=\int_{\varphi(x)^{1/\alpha}}^\infty \frac{\vol(\frac{B}{r})}{\vol(K_\varphi)}\cdot p(r)\diff r\\
		&=\vol(B)\cdot \int_{\varphi(x)^{1/\alpha}}^\infty \frac{p(r)r^{-n}}{\vol(K_\varphi)}\diff r.
	\end{align*}
	So the density of $rV$ at $x$ is $\int_{\varphi(x)^{1/\alpha}}^\infty \frac{p(r)r^{-n}}{\vol(K_\varphi)}\diff r$. In order to match it with $\frac{1}{Z}\e^{-\varphi}$, we have
	\begin{align*}
		\int_{\varphi(x)^{1/\alpha}}^\infty \frac{p(r)r^{-n}}{\vol(K_\varphi)}\diff r
		&= \frac{\e^{-\varphi(x)}}{\Gamma(\tfrac{n}{\alpha}+1) \cdot \vol(K_\varphi)}\\
		\int_{\varphi(x)^{1/\alpha}}^\infty p(r)r^{-n}\diff r &=\frac{\e^{-\varphi(x)}}{\Gamma(\tfrac{n}{\alpha}+1)}
	\end{align*}
	Let $\varphi(x)^{1/\alpha}=u$, we have
	\[
		\int_u^\infty p(r)r^{-n}\diff r = \frac{1}{\Gamma(\tfrac{n}{\alpha}+1)}\cdot\e^{-u^\alpha}.
	\]
	Taking derivative with respect to $u$, we have
	\[p(u)u^{-n} = \frac{1}{\Gamma(\tfrac{n}{\alpha}+1)}\cdot\e^{-u^\alpha}\cdot \alpha u^{\alpha-1}.\]
	It's straightforward to show that if $t\sim \Gamma(\frac{n}{\alpha}+1,1)$, then $t^{\frac{1}{\alpha}}$ has the above density $p(u)$.
\end{proof}
A simple but useful corollary of \Cref{thm:uniform} is
\begin{corollary}\label{cor:unif}
	Let $t\sim \Gamma(\frac{n}{\alpha}+1,1)$.  Let $U$ has uniform distribution over $\partial K_\varphi$, independently from $t$ and $r$ with density $nx^{n-1}$ over $[0,1]$, independent from $t$ and $U$. Then $t^{\frac{1}{\alpha}}\cdot r\cdot U$ has density $\frac{1}{Z}\e^{-\varphi}$.
\end{corollary}
We commented that we can compute the isotropic constants for $\ell_p$ balls. The rest of the section is dedicated to this kind of results.

Let $\mu$ be a log-concave probability measure on $\R^n$ with density $f_\mu:\R^n\to\R_{\geqslant0}$.
The isotropic constant of $\mu$ is defined by (see e.g. \cite{giannopoulosgeometry})
\[L_\mu=\big(\textstyle\sup_{x\in\R^n}f_\mu(x)\big)^{\frac{1}{n}}\cdot \big(\det\mathrm{Cov}(\mu)\big)^{\frac{1}{2n}}.\]
As a special case, when $\mu$ is the uniform distribution over the convex body $K$, the corresponding isotropic constant is denoted by $L_K$ and has expression
\[L_K=\vol(K)^{-\frac{1}{n}}\cdot \big(\det\mathrm{Cov}(\mu)\big)^{\frac{1}{2n}}.\]
For homogeneous and convex $\varphi:\R^n\to\R$, we use $L_\varphi$ to denote the isotropic constant of its associated probability distribution, i.e. the one with density $\frac{1}{Z_\varphi}\e^{-\varphi(x)}$. With the help of \Cref{thm:uniform}, we can relate $L_\varphi$ to the isotropic constant of its unit ball $L_{K_\varphi}$.
\begin{lemma} \label{lem:relate}
	\[
		L_\varphi=
		\frac{[\Gamma(\frac{n}{\alpha}+1+\frac{2}{\alpha})]^{\frac{1}{2}}}{[\Gamma(\frac{n}{\alpha}+1)]^{\frac{1}{2}+{\frac{1}{n}}}}
		\cdot L_{K_\varphi}
	\]
\end{lemma}
\begin{proof}
	\begin{align*}
		\mathrm{Cov}(X_\varphi)
		&=\E [X_\varphi X_\varphi^T]
		=\E [t^{\frac{2}{\alpha}}\cdot V_\varphi V_\varphi^T]
		=\E t^{\frac{2}{\alpha}}\cdot \E[V_\varphi V_\varphi^T]
		=\E t^{\frac{2}{\alpha}}\cdot 
		\mathrm{Cov}(V_\varphi)
	\end{align*}
	\[
	\det \mathrm{Cov}(X_\varphi)=(\E t^{\frac{2}{\alpha}})^n\cdot \det \mathrm{Cov}(V_\varphi)
	\]
	\begin{align*}
		L_\varphi
		&= Z_\varphi^{-\frac{1}{n}} \cdot (\det \mathrm{Cov}(X_\varphi))^{\frac{1}{2n}}\\
		&= Z_\varphi^{-\frac{1}{n}} \cdot 
		(\E t^{\frac{2}{\alpha}})^{\frac{1}{2}}\cdot (\det \mathrm{Cov}(V_\varphi))^{\frac{1}{2n}}\\
		&=\Gamma(\tfrac{n}{\alpha}+1)^{-\frac{1}{n}} \cdot \vol(K_\varphi)^{-\frac{1}{n}}\cdot
		(\E t^{\frac{2}{\alpha}})^{\frac{1}{2}}\cdot (\det \mathrm{Cov}(V_\varphi))^{\frac{1}{2n}}
		&&(\text{\Cref{lem:partition}})\\
		&=(\E t^{\frac{2}{\alpha}})^{\frac{1}{2}}\cdot \Gamma(\tfrac{n}{\alpha}+1)^{-\frac{1}{n}}
		\cdot L_{K_\varphi}
	\end{align*}
	By \Cref{lem:gamma_moment}, $\E t^{\frac{2}{\alpha}} =\frac{\Gamma(\frac{n}{\alpha}+1+\frac{2}{\alpha})}{\Gamma(\frac{n}{\alpha}+1)}$. So
	\begin{align*}
		L_\varphi
		&=\left(\frac{\Gamma(\frac{n}{\alpha}+1+\frac{2}{\alpha})}{\Gamma(\frac{n}{\alpha}+1)}\right)^{\frac{1}{2}}\cdot \Gamma(\tfrac{n}{\alpha}+1)^{-\frac{1}{n}}
		\cdot L_{K_\varphi}\\
		&=\frac{[\Gamma(\frac{n}{\alpha}+1+\frac{2}{\alpha})]^{\frac{1}{2}}}{[\Gamma(\frac{n}{\alpha}+1)]^{\frac{1}{2}+{\frac{1}{n}}}}
		\cdot L_{K_\varphi}
	\end{align*}
\end{proof}

The last result is a sufficient condition of the Fisher information being a scalar matrix.
\begin{lemma} \label{lem:uncondition}
	If $\varphi$ is invariant under the action of $\{\pm1\}^n$ and cyclic group of size $n$, i.e.
	\begin{enumerate}
		\item[\textup{1.}] $\varphi(\pm x_1,\pm x_2,\ldots, \pm x_n) = \varphi( x_1, x_2,\ldots,  x_n)$
		\item[\textup{2.}] $\varphi( x_1, x_2,\ldots, x_{n-1}, x_n) = \varphi( x_2, x_3,\ldots,  x_n, x_1)$
	\end{enumerate}
	then $\Fisher_\varphi=\frac{1}{n}\E\|\nabla\varphi\|_2^2\cdot I$.
\end{lemma}
\begin{proof}[Proof of \Cref{lem:uncondition}]
	First we use the symmetry to show $\Fisher_\varphi=cI$ for some $c$.
	\begin{align*}
		\varphi( x_1, x_2,\ldots, x_n)&=\varphi( -x_1, x_2,\ldots, x_n)\\
		\partial_1\varphi( x_1, x_2,\ldots, x_n)&=-\partial_1\varphi( -x_1, x_2,\ldots, x_n)\\
		\partial_2\varphi( x_1, x_2,\ldots, x_n)&=\partial_2\varphi( -x_1, x_2,\ldots, x_n)
	\end{align*}
	This shows $\partial_1\varphi\cdot\partial_2\varphi$ is an odd function of $x_1$. On the other hand, we know the density $\e^{-\varphi}$ is an even function of $x_1$. So we conclude that $\E[\partial_1\varphi\cdot\partial_2\varphi]=0$. Similarly, we can show that for any $i\neq j, \E[\partial_i\varphi\cdot\partial_j\varphi]=0$. This shows that $\Fisher_\varphi$ is a diagonal matrix.

	By cyclic symmetry, we have
	\begin{align*}
		\varphi( x_1, x_2,\ldots, x_{n-1}, x_n) &= \varphi( x_2, x_3,\ldots,  x_n, x_1)\\
		\partial_1\varphi( x_1, x_2,\ldots, x_{n-1}, x_n) &= \partial_n\varphi( x_2, x_3,\ldots,  x_n, x_1)\\
		\partial_1\varphi( x_1, x_2,\ldots, x_{n-1}, x_n)^2\e^{-\varphi( x_1, x_2,\ldots, x_{n-1}, x_n)}
		&=\partial_n\varphi( x_2, x_3,\ldots,  x_n, x_1)^2\e^{-\varphi( x_1, x_2,\ldots, x_{n-1}, x_n)}\\
		&=\partial_n\varphi( y_1, y_2,\ldots, y_{n-1}, y_n)^2\e^{-\varphi( y_n, y_1,\ldots, y_{n-1})}\\
		&=\partial_n\varphi( y_1, y_2,\ldots, y_{n-1}, y_n)^2\e^{-\varphi( y_1, y_2,\ldots, y_{n})}
	\end{align*}
	This shows $\E\partial_1\varphi^2=\cdots =\E\partial_n\varphi^2$. Hence $\Fisher_\varphi=cI$ for some $c$.
	\begin{align*}
		\tr\Fisher_\varphi = \tr\E[\nabla\varphi\nabla\varphi^T]=\E[\tr\nabla\varphi\nabla\varphi^T]=\E[\tr\nabla\varphi^T\nabla\varphi]=\E\|\nabla\varphi\|_2^2.
	\end{align*}
	On the other hand, $\tr\Fisher_\varphi = \tr cI = cn$, so $c = \frac{1}{n}\E\|\nabla\varphi\|_2^2$.
\end{proof}


\subsection{Calculation for Norm Powers} 
\label{sub:calculation_for_norm_powers}

\var*
We divide the proof into two parts, one for each of the equations.
\begin{proof}[Proof of \Cref{thm:variance} (variance part)]
By \Cref{thm:uniform}, let $t\sim\Gamma(\frac{n}{\alpha}+1,1)$ random variable and $V_p$ has uniform distribution over the $\ell_p$ unit ball $K_p$.
\begin{align}\label{eq:pa}
	\E\|X_{p,\alpha}\|_2^2 = \E t^{\frac{2}{\alpha}}\cdot\E\|V_p\|_2^2 = \frac{\Gamma(\frac{n}{\alpha}+1+\frac{2}{\alpha})}{\Gamma(\frac{n}{\alpha}+1)}\cdot\E\|V_p\|_2^2
\end{align}
Setting $\alpha=p$ yields
\begin{equation}\label{eq:pp}
	\E\|X_{p,p}\|_2^2 = \frac{\Gamma(\frac{n}{p}+1+\frac{2}{p})}{\Gamma(\frac{n}{p}+1)}\cdot\E\|V_p\|_2^2
\end{equation}
The reason we do this is that $\E\|X_{p,p}\|_2^2$ can be computed explicitly. In fact,
\begin{align*}
	\E\|X_{p,p}\|_2^2
	&=\frac{1}{Z_n}\int \sum x_i^2\cdot \e^{-\sum |x_i|^p}\diff x\\
	&=\frac{n}{Z_n}\int x_1^2\cdot \e^{-\sum |x_i|^p}\diff x\\
	&=\frac{n}{Z_n}\int x_1^2\cdot \e^{-|x_1|^p}\diff x_1\cdot Z_{n-1}
\end{align*}
where $Z_n=\int \e^{-\sum |x_i|^p}\diff x$.
We know by \Cref{lem:vol} that
\[Z_n=\Gamma(\tfrac{n}{p}+1) \cdot \vol(K_p) = 2^n\cdot{\Gamma(\tfrac{1}{p}+1)^n}\]
and
\begin{align*}
	\int_{-\infty}^{+\infty} x^2\cdot \e^{-|x|^p}\diff x
	&=2\int_{0}^{\infty} x^2\cdot \e^{-x^p}\diff x
	=2\int_{0}^{\infty} y^{\tfrac{2}{p}}\cdot \e^{-y}\cdot \tfrac{1}{p}y^{\tfrac{1}{p}-1}\diff y
	=\tfrac{2}{p}\int_{0}^{\infty} y^{\tfrac{3}{p}-1}\cdot \e^{-y}\diff y
	=\tfrac{2}{p}\cdot\Gamma(\tfrac{3}{p})
\end{align*}
So
\begin{align*}
	\E\|X_{p,p}\|_2^2
	&=\frac{n}{Z_n}\int x_1^2\cdot \e^{-|x_1|^p}\diff x_1\cdot Z_{n-1}\\
	&=n\cdot\tfrac{2}{p}\cdot\Gamma(\tfrac{3}{p})\cdot \frac{Z_{n-1}}{Z_n}\\
	&=\frac{n\cdot\tfrac{2}{p}\cdot\Gamma(\tfrac{3}{p})}{2\Gamma(\tfrac{1}{p}+1)}\\
	&=\frac{n\cdot\tfrac{2}{p}\cdot\Gamma(\tfrac{3}{p})}{\tfrac{2}{p}\Gamma(\tfrac{1}{p})}
	=n\cdot\frac{\Gamma(\tfrac{3}{p})}{\Gamma(\frac{1}{p})}
\end{align*}
Plugging this into \eqref{eq:pp}, we have
\begin{align*}
	\E\|V_p\|_2^2
	&=\frac{\Gamma(\frac{n}{p}+1)}{\Gamma(\frac{n}{p}+1+\frac{2}{p})}\cdot\E\|X_{p,p}\|_2^2\\
	&=\frac{\Gamma(\frac{n}{p}+1)}{\Gamma(\frac{n}{p}+1+\frac{2}{p})}\cdot
	n\cdot\frac{\Gamma(\tfrac{3}{p})}{\Gamma(\frac{1}{p})}.
\end{align*}
Using this in \eqref{eq:pa},
\begin{align*}
	\E\|X_{p,\alpha}\|_2^2
	&= \frac{\Gamma(\frac{n}{\alpha}+1+\frac{2}{\alpha})}{\Gamma(\frac{n}{\alpha}+1)}\cdot\E\|V_p\|_2^2\\
	&=\frac{\Gamma(\frac{n}{\alpha}+1+\frac{2}{\alpha})}{\Gamma(\frac{n}{\alpha}+1)}\cdot \frac{\Gamma(\frac{n}{p}+1)}{\Gamma(\frac{n}{p}+1+\frac{2}{p})}\cdot n\cdot\frac{\Gamma(\tfrac{3}{p})}{\Gamma(\frac{1}{p})}
\end{align*}
In order to study the asymptotics of $\E\|X_{p,\alpha}\|_2^2$ as $n\to\infty$, recall Stirling's formula
\[\Gamma(z+1)\sim \sqrt{2\pi z}\left(\frac{z}{\e}\right)^z.\]
So we have
\begin{align*}
	\frac{\Gamma(\frac{n}{\alpha}+1+\frac{2}{\alpha})}{\Gamma(\frac{n}{\alpha}+1)}
	&\sim \frac{\sqrt{\frac{n+2}{\alpha}}\cdot\left(\frac{n+2}{\alpha\e}\right)^{\frac{n+2}{\alpha}}}{\sqrt{\frac{n}{\alpha}}\cdot\left(\frac{n}{\alpha\e}\right)^{\frac{n}{\alpha}}}\\
	&\sim \left(\frac{n+2}{\alpha\e}\cdot \frac{\alpha\e}{n}\right)^{\frac{n}{\alpha}}\cdot \left(\frac{n+2}{\alpha\e}\right)^{\frac{2}{\alpha}}\\
	&= \left(1+\frac{2}{n}\right)^{\frac{n}{2}\cdot \frac{2}{\alpha}}\cdot \left(\frac{n+2}{\alpha\e}\right)^{\frac{2}{\alpha}}\\
	&\sim \left(\frac{n+2}{\alpha}\right)^{\frac{2}{\alpha}}
\end{align*}
Hence
\begin{align*}
	\E\|X_{p,\alpha}\|_2^2
	\sim \left(\frac{n+2}{\alpha}\right)^{\frac{2}{\alpha}}/ \left(\frac{n+2}{p}\right)^{\frac{2}{p}}\cdot n\cdot\frac{\Gamma(\tfrac{3}{p})}{\Gamma(\frac{1}{p})}
	\sim n^{\frac{2}{\alpha}-\frac{2}{p}+1} \cdot \alpha^{-\frac{2}{\alpha}}\cdot p^{\frac{2}{p}}\cdot\frac{\Gamma(\tfrac{3}{p})}{\Gamma(\frac{1}{p})}
\end{align*}
\end{proof}
Before we proceed to the proof of the Fisher information part of \Cref{thm:variance}, we derive the isotropic constants results as promised, using \Cref{lem:relate} and the variance part of \Cref{thm:variance}.
\begin{corollary}\label{cor:iso_unif}
	The isotropic constant of $n$-dimensional $\ell_p$ ball is
	\[
		L_{K_p}^2=
		\frac{p^2}{4}\cdot
		\frac{\Gamma(\tfrac{3}{p})}{[\Gamma(\frac{1}{p})]^3}\cdot
		\frac{[\Gamma(\frac{n}{p}+1)]^{1+\frac{2}{n}}}{\Gamma(\frac{n}{p}+1+\frac{2}{p})}
	\]
\end{corollary}
\begin{proof}[Proof of \Cref{cor:iso_unif}]
For a general convex body $K$, let $V_K$ be a random vector with the uniform distribution over $K$. Recall that the isotropic constant of $K$ is
\[L_K=\vol(K)^{-\frac{1}{n}}\cdot \big(\det\mathrm{Cov}(V_K)\big)^{\frac{1}{2n}}.\]
Now we focus on the unit ball of $\ell_p$ norm $K_p$. The corresponding random vector is denoted by $V_p$
By a symmetry argument similar to \Cref{lem:uncondition}, we have that
\[\mathrm{Cov}(V_p)=\frac{1}{n}\cdot\E\|V_p\|_2^2\cdot I_{n\times n}.\]
Combining this and \Cref{lem:vol}, we have
\begin{align*}
	L_{K_p}^2&=\vol(K_p)^{-\frac{2}{n}}\cdot {\tfrac{1}{n}\E\|V_p\|_2^2}\\
	&=\left(2^n\cdot\frac{\Gamma(\frac{1}{p}+1)^n}{\Gamma(\frac{n}{p}+1)}\right)^{-\frac{2}{n}}
	\cdot \frac{\Gamma(\frac{n}{p}+1)}{\Gamma(\frac{n}{p}+1+\frac{2}{p})}\cdot
	\frac{\Gamma(\tfrac{3}{p})}{\Gamma(\frac{1}{p})}\\
	&=\frac{1}{4}\cdot\frac{\Gamma(\frac{n}{p}+1)^{\frac{2}{n}}}{\Gamma(\frac{1}{p}+1)^2}\cdot \frac{\Gamma(\frac{n}{p}+1)}{\Gamma(\frac{n}{p}+1+\frac{2}{p})}\cdot
	\frac{\Gamma(\tfrac{3}{p})}{\Gamma(\frac{1}{p})}\\
	&=\frac{1}{4}\cdot\frac{\Gamma(\frac{n}{p}+1)^{\frac{2}{n}}}{\frac{1}{p^2}\Gamma(\frac{1}{p})^2}\cdot \frac{\Gamma(\frac{n}{p}+1)}{\Gamma(\frac{n}{p}+1+\frac{2}{p})}\cdot
	\frac{\Gamma(\tfrac{3}{p})}{\Gamma(\frac{1}{p})}\\
	&=
		\frac{p^2}{4}\cdot
		\frac{\Gamma(\tfrac{3}{p})}{[\Gamma(\frac{1}{p})]^3}\cdot
		\frac{[\Gamma(\frac{n}{p}+1)]^{1+\frac{2}{n}}}{\Gamma(\frac{n}{p}+1+\frac{2}{p})}
\end{align*}
\end{proof}
\begin{corollary} \label{cor:iso_phi}
	When $\varphi(x)=\|x\|_p^\alpha$, 
	\[
		L_{p,\alpha}^2=
		\frac{p^2}{4}\cdot
		\frac{\Gamma(\tfrac{3}{p})}{[\Gamma(\frac{1}{p})]^3}\cdot
		\frac{\Gamma(\frac{n}{\alpha}+1+\frac{2}{\alpha})}{\Gamma(\frac{n}{p}+1+\frac{2}{p})}\cdot 
		\left(\frac{\Gamma(\frac{n}{p}+1)}{\Gamma(\frac{n}{\alpha}+1)}\right)^{1+\frac{2}{n}}
	\]
\end{corollary}
\begin{proof}[Proof of \Cref{cor:iso_phi}]
Directly follows from the above result and \Cref{lem:relate}.
\end{proof}

Now we turn our attention back to the proof of \Cref{thm:variance}.
\begin{proof}[Proof of \Cref{thm:variance} (variance part)]
	By \Cref{lem:uncondition}, $\Fisher_{p,\alpha}=\|\Fisher_{p,\alpha}\|_2\cdot I_{n\times n}$ and $\|\Fisher_{p,\alpha}\|_2=\frac{1}{n}\E\|\nabla\varphi\|_2^2$ where $\varphi(x)=\|x\|_p^\alpha$. In this case the gradient has an explicit expression:
	\begin{align*}
		[\nabla\varphi(x)]_i&=\alpha\cdot \big(\sum|x_i|^{p}\big)^{\frac{\alpha}{p}-1}\cdot |x_i|^{p-1}\cdot \mathrm{sgn} (x_i) = \alpha\|x\|_{p}^{\alpha-p}\cdot |x_i|^{p-1}\cdot \mathrm{sgn} (x_i)\\
		\|\nabla\varphi(x)\|_2^2&=\alpha^2\sum|x_i|^{2p-2}\cdot \big(\sum|x_i|^{p}\big)^{\frac{2\alpha}{p}-2}=\alpha^2\|x\|_{2p-2}^{2p-2}\cdot\|x\|_{p}^{2\alpha-2p}
	\end{align*}
	By \Cref{cor:unif}, $X_{p,\alpha}\stackrel{\diff}{=}t^{\frac{1}{\alpha}}\cdot r\cdot U_p$, 
	where $t\sim \Gamma(\frac{n}{\alpha}+1,1)$, $U$ has uniform distribution over $\partial K_p$and $r$ has density $nx^{n-1}$ over $[0,1]$.
	\begin{align*}
		\|\nabla\varphi\|_2^2
		&=\alpha^2\|t^{\frac{1}{\alpha}}rU\|_{2p-2}^{2p-2}\cdot\|t^{\frac{1}{\alpha}}rU\|_{p}^{2\alpha-2p}\\
		&=\alpha^2t^{\frac{2\alpha-2}{\alpha}}r^{2\alpha-2}\cdot\|U\|_{2p-2}^{2p-2}\cdot\|U\|_{p}^{2\alpha-2p}
	\end{align*}
	Since $t,r,U_p$ are independent and $\|U_p\|_p=1$, we have
	\begin{equation}\label{eq:fish1}		
		\|\Fisher_{p,\alpha}\|_2=\frac{1}{n}\E\|\nabla\varphi\|_2^2
		=\frac{\alpha^2}{n}\cdot \E t^{\frac{2\alpha-2}{\alpha}}\cdot \E r^{2\alpha-2}\cdot \E\|U\|_{2p-2}^{2p-2}
	\end{equation}
		By \Cref{lem:gamma_moment},
		\[
		\E t^{\frac{2\alpha-2}{\alpha}}
		=\frac{\Gamma(\frac{n}{\alpha}+1+\frac{2\alpha-2}{\alpha})}{\Gamma(\frac{n}{\alpha}+1)}=\frac{\Gamma(\frac{n+2\alpha-2}{\alpha}+1)}{\Gamma(\frac{n}{\alpha}+1)}.
		\]
		The moment of $r$ can be computed directly
		\[
		\E r^{2\alpha-2}
		= \int_0^1 r^{2\alpha-2}\cdot nr^{n-1}\diff r = \frac{n}{n+2\alpha-2}.
		\]
	Plugging into \eqref{eq:fish1}, we have
	\begin{align*}\label{eq:star}	
		\|\Fisher_{p,\alpha}\|_2
		&=\frac{\alpha^2}{n}\cdot \E t^{\frac{2\alpha-2}{\alpha}}\cdot \E r^{2\alpha-2}\cdot \E\|U_p\|_{2p-2}^{2p-2}\\
		&= \frac{\alpha^2}{n}\cdot\frac{\Gamma(\frac{n+2\alpha-2}{\alpha}+1)}{\Gamma(\frac{n}{\alpha}+1)}\cdot\frac{n}{n+2\alpha-2}\cdot\E\|U_p\|_{2p-2}^{2p-2}\\
		&= \frac{\alpha^2}{n}\cdot\frac{\frac{n+2\alpha-2}{\alpha}\cdot\Gamma(\frac{n+2\alpha-2}{\alpha})}{\frac{n}{\alpha}\cdot\Gamma(\frac{n}{\alpha})}\cdot\frac{n}{n+2\alpha-2}\cdot\E\|U_p\|_{2p-2}^{2p-2}\\
		&= \frac{\alpha^2}{n}\cdot\frac{\Gamma(\frac{n+2\alpha-2}{\alpha})}{\Gamma(\frac{n}{\alpha})}\cdot\E\|U_p\|_{2p-2}^{2p-2}\tag{$*$}
	\end{align*}
	On the other hand, when $p=\alpha$, we have
	\begin{align*}
		\|\nabla\varphi(x)\|_2^2=p^2\sum|x_i|^{2p-2}.
	\end{align*}
	In this case, $x$ has joint density $\propto \e^{-\|x\|_p^p}=\e^{-\sum|x_i|^{p}}$.
	Let $Y_1,Y_2,\ldots, Y_n$ be i.i.d. random variables with density $\propto\e^{-|y|^p}$. Then we have
	\[
	\|\Fisher_{p,p}\|_2=\frac{1}{n}\E\|\nabla\varphi\|_2^2=\frac{1}{n}\cdot p^2\cdot \E\sum|Y_i|^{2p-2}=p^2\E\big[|Y_i|^{2p-2}\big]
	\]
	Let $z=\int_{-\infty}^{+\infty}\e^{-|y|^p}\diff y$.
	\begin{align*}
		\E\big[|Y_i|^{2p-2}\big]
		&=\frac{1}{z}\int_{-\infty}^{+\infty}|y|^{2p-2}\cdot\e^{-|y|^p}\diff y\\
		&=\frac{\int_{0}^{+\infty}y^{2p-2}\cdot\e^{-y^p}\diff y}{\int_{0}^{+\infty}\e^{-y^p}\diff y}\\
		&=\frac{\int_{0}^{+\infty}x^{\frac{2p-2}{p}}\cdot x^{\frac{1}{p}-1}\cdot\e^{-x}\diff x}{\int_{0}^{+\infty}x^{\frac{1}{p}-1}\cdot\e^{-x}\diff x}\\
		&=\frac{\Gamma(2-\frac{1}{p})}{\Gamma(\frac{1}{p})}
	\end{align*}
	Relating to \eqref{eq:star} in the special case of $\alpha=p$,
	\begin{align*}
	\frac{p^2}{n}\cdot\frac{\Gamma(\frac{n+2p-2}{p})}{\Gamma(\frac{n}{p})}\cdot\E\|U_p\|_{2p-2}^{2p-2}=\|\Fisher_{p,p}\|_2=p^2\cdot \frac{\Gamma(2-\frac{1}{p})}{\Gamma(\frac{1}{p})}
	\end{align*}
	Hence
	\begin{align*}
	\E\|U_p\|_{2p-2}^{2p-2}=\frac{n}{p^2}\cdot\frac{\Gamma(\frac{n}{p})}{\Gamma(\frac{n+2p-2}{p})}\cdot\|\Fisher_{p,p}\|_2=n\cdot\frac{\Gamma(\frac{n}{p})}{\Gamma(\frac{n+2p-2}{p})}\cdot \frac{\Gamma(2-\frac{1}{p})}{\Gamma(\frac{1}{p})}
	\end{align*}
	Using \eqref{eq:star} again,

	\begin{align*}
		\|\Fisher_{p,\alpha}\|_2
		&= \frac{\alpha^2}{n}\cdot\frac{\Gamma(\frac{n+2\alpha-2}{\alpha})}{\Gamma(\frac{n}{\alpha})}\cdot\E\|U_p\|_{2p-2}^{2p-2}\\
		&= \frac{\alpha^2}{n}\cdot\frac{\Gamma(\frac{n+2\alpha-2}{\alpha})}{\Gamma(\frac{n}{\alpha})}\cdot n\cdot\frac{\Gamma(\frac{n}{p})}{\Gamma(\frac{n+2p-2}{p})}\cdot \frac{\Gamma(2-\frac{1}{p})}{\Gamma(\frac{1}{p})}\\
		&=\alpha^2\cdot \frac{\Gamma(\frac{n+2\alpha-2}{\alpha})}{\Gamma(\frac{n}{\alpha})}\cdot \frac{\Gamma(\frac{n}{p})}{\Gamma(\frac{n+2p-2}{p})}\cdot \frac{\Gamma(2-\frac{1}{p})}{\Gamma(\frac{1}{p})}
	\end{align*}

In order to study the asymptotics, using Stirling's formula again,
\begin{align*}
	\frac{\Gamma\left(\frac{2\alpha-2+n}{\alpha}\right)}{\Gamma(\tfrac{n}{\alpha}+1)}
	&\sim \frac{\sqrt{\frac{n+\alpha-2}{\alpha}}\cdot\left(\frac{n+\alpha-2}{\alpha\e}\right)^{\frac{n+\alpha-2}{\alpha}}}{\sqrt{\frac{n}{\alpha}}\cdot\left(\frac{n}{\alpha\e}\right)^{\frac{n}{\alpha}}}\\
	&\sim \left(\frac{n+\alpha-2}{\alpha\e}\cdot \frac{\alpha\e}{n}\right)^{\frac{n}{\alpha}}\cdot \left(\frac{n+\alpha-2}{\alpha\e}\right)^{\frac{\alpha-2}{\alpha}}\\
	&= \left(1+\frac{\alpha-2}{n}\right)^{\frac{n}{\alpha}}\cdot \left(\frac{n+\alpha-2}{\alpha\e}\right)^{\frac{\alpha-2}{\alpha}}\\
	&\sim \e^{\frac{\alpha-2}{\alpha}}\cdot \left(\frac{n+\alpha-2}{\alpha\e}\right)^{\frac{\alpha-2}{\alpha}}\\
	&\sim \left(\frac{n+\alpha-2}{\alpha}\right)^{\frac{\alpha-2}{\alpha}}\\
	&\sim n^{1-\frac{2}{\alpha}}\cdot \alpha^{-1+\frac{2}{\alpha}}
\end{align*}
\begin{align*}
	\|\Fisher_{p,\alpha}\|_2
		&=\alpha^2\cdot \frac{\Gamma(\frac{n+2\alpha-2}{\alpha})}{\Gamma(\frac{n}{\alpha})}\cdot \frac{\Gamma(\frac{n}{p})}{\Gamma(\frac{n+2p-2}{p})}\cdot \frac{\Gamma(2-\frac{1}{p})}{\Gamma(\frac{1}{p})}\\
		&=\alpha^2\cdot \frac{\Gamma(\frac{n+2\alpha-2}{\alpha})}{\Gamma(\frac{n}{\alpha}+1)}\cdot \frac{n}{\alpha}\cdot \frac{\Gamma(\frac{n}{p})}{\Gamma(\frac{n+2p-2}{p})}\cdot \frac{p}{n}\cdot \frac{\Gamma(2-\frac{1}{p})}{\Gamma(\frac{1}{p})}\\
		&\sim \alpha p\cdot n^{1-\frac{2}{\alpha}}\cdot \alpha^{-1+\frac{2}{\alpha}}\cdot n^{-1+\frac{2}{p}}\cdot p^{1-\frac{2}{p}}\cdot \frac{\Gamma(2-\frac{1}{p})}{\Gamma(\frac{1}{p})}\\
&\sim\alpha^{\frac{2}{\alpha}}\cdot p^{2-\frac{2}{p}}\cdot\frac{\Gamma(2-\frac{1}{p})}{\Gamma(\frac{1}{p})} \cdot n^{\frac{2}{p}-\frac{2}{\alpha}}
	\end{align*}
This finishes the entire proof of \Cref{thm:variance}.
\end{proof}

Next we turn our attention to $\ell_\infty$ error. The counterpart for the first half of \Cref{thm:variance} is the following lemma.

\section{Proof of \Cref{lem:normpower}}
\label{sec:proof-case-x_palpha}

Without the loss of generality, we assume $\mu=1$ in the proof. 

By Lemma~\ref{thm:variance}, we then have $t_n=\frac{1}{\mu}\cdot\sqrt{\|\Fisher_{\varphi_n}\|_2}\asymp n^{\frac{1}{\alpha}-\frac{1}{p}}$, so the rescaled  $\tilde{\varphi}_n:\R^n\to\R$  has the form
\[
\tilde\varphi_n(x) = c_{p,\alpha} n^{1 - \frac{\alpha}{p}} \|x\|_p^{\alpha},
\]
where $c_{p,\alpha}=\alpha^{-1}\cdot p^{-\alpha+\frac{\alpha}{p}}\cdot\left(\frac{\Gamma(2-\frac{1}{p})}{\Gamma(\frac{1}{p})}\right)^{-\frac{\alpha}{2}}$.

\subsection{Lemmas}
We first state a few auxiliary lemmas.

\begin{lemma}\label{lem:norm}

\[
\|X\|_p = ( (\frac{1}{\alpha})^{1/\alpha} + o_{\P}(1)) \cdot n^{\frac1p}.
\]

\end{lemma}

\begin{lemma}\label{lem:badregion} If $p>1, \alpha>0$, then
\[
\sum_{i: |X_i| \le 2 |v_i| } ( |X_i + v_i|^p - |X_i|^p) = o_{\P}(1).
\]
\[
\sum_{i: |X_i| \le 2 |v_i| }p \sign(X_i) |X_i|^{p-1} v_i  = o_{\P}(1).
\]
$$
\sum_{i: |X_i| < 2 |v_i| } \frac{p(p-1)}{2}  |X_i|^{p-2} v_i^2=o_{\P}(1)
$$

\end{lemma}

\subsection{Main proof}


We will first prove the case where $p>1$. The proof of the corner case where $p=1$ will be given in Section~\ref{pf:p1}


\textbf{Verification of Condition D1.}

Recall that
\[
\tilde\varphi_n(x) = c_{p,\alpha} n^{1 - \frac{\alpha}{p}} \|x\|_p^{\alpha},
\]
where $c_{p,\alpha}=\alpha^{-1}\cdot p^{-\alpha+\frac{\alpha}{p}}\cdot\left(\frac{\Gamma(2-\frac{1}{p})}{\Gamma(\frac{1}{p})}\right)^{-\frac{\alpha}{2}}$.

We then have \begin{align*}
\Proj_v^\varphi(X)=&\varphi(X+v)-\varphi(X)-\frac{1}{2}v^T\Fisher_\varphi v\\
=&c_{p,\alpha} n^{1 - \frac{\alpha}{p}}(\|X+v\|_p^\alpha -\|X\|_p^\alpha-\frac{1}{2} )
\end{align*}
$$
\nabla \phi(X)=c_{p,\alpha} n^{1 - \frac{\alpha}{p}}\cdot\alpha\|X\|^{\alpha-p}_p \cdot \sign(X)\odot X^{\odot(p-1)}
$$

Now let us consider $$
P_v^\varphi(X)-v^T\nabla\varphi(X)=c_{p,\alpha} n^{1 - \frac{\alpha}{p}}(\|X+v\|_p^\alpha -\|X\|_p^\alpha-\frac{1}{2}-\langle \alpha\|X\|^{\alpha-p}_p \cdot \sign(X)\odot X^{\odot(p-1)},v\rangle ).
$$
Then
\begin{align*}
&\|X+v\|_p^\alpha -\|X\|_p^\alpha-\langle\alpha\|X\|^{\alpha-p}_p \cdot \sign(X)\odot  X^{\odot(p-1)}, v\rangle\\
\asymp& (1 + o_{\P}(1)) \frac{\alpha \|X\|_p^{\alpha - p}}{p} \left( \|X + v\|_p^p - \|X\|_p^p -p\sign(X)\odot X^{\odot(p-1)}\odot v\right)\\
\asymp& \frac{1}{p}\cdot\alpha^{\frac{p}{\alpha}}\cdot n^{\frac{\alpha-p}{p}} \left( \|X + v\|_p^p - \|X\|_p^p -p\sign(X)\odot\cdot X^{\odot(p-1)}\odot v\right)\\
\asymp& \frac{1}{p}\cdot\alpha^{\frac{p}{\alpha}}\cdot n^{\frac{\alpha-p}{p}} \left( \sum_{i=1}^n ( |X_i + v_i|^p - |X_i|^p-p\sign(X_i)\cdot X_i^{p-1}) v_i\right)
\end{align*}

To prove (1), it suffices to show $\left( \sum_{i=1}^n ( |X_i + v_i|^p - |X_i|^p)-p\sign(X_i)\cdot X_i^{p-1} v_i-\frac{1}{2}\right)=o_{\P}(1)$. We expand this expression as
\begin{align*}
&\sum_{i=1}^n ( |X_i + v_i|^p - |X_i|^p-p\sign(X_i)\cdot X_i^{p-1}v_i) \\
=&\sum_{i: |X_i| > 2 |v_i| } ( |X_i + v_i|^p - |X_i|^p-p\sign(X_i)\cdot X_i^{p-1}v_i) + \sum_{i: |X_i| \le 2 |v_i| } ( |X_i + v_i|^p - |X_i|^p)-\sum_{i: |X_i| \le 2 |v_i| }p \sign(X_i) |X_i|^{p-1} v_i  \\
\end{align*}
When $|X_i| > 2 |v_i|$, note that 
\[
|X_i + v_i|^p = |X_i|^p (1 + v_i/X_i)^p =  |X_i|^p \left( 1 + p\frac{v_i}{X_i} + \frac{p(p-1)}{2} \frac{v_i^2}{X_i^2}  + O\left(\frac{v_i^3}{X_i^3} \right) \right).
\]
Combing with Lemma~\ref{lem:badregion}, we then have
\begin{align*}
&\sum_{i=1}^n ( |X_i + v_i|^p - |X_i|^p-p\sign(X_i)\cdot X_i^{p-1}v_i) \\
=&\sum_{i: |X_i| > 2 |v_i| }  |X_i|^p \left( \frac{p(p-1)}{2} \frac{v_i^2}{X_i^2}  + O\left(\frac{v_i^3}{X_i^3} \right) \right)+ o_{\P}(1)\\
\end{align*}

In order to verify Condition D1, we need to show: 1). $\sum_{i: |X_i| > 2 |v_i| } \frac{p(p-1)}{2}  |X_i|^{p-2} v_i^2$ converges to a constant; 2). Show the third order term is vanishing, that is,
$
\sum_{i} |X_i|^{p-3} |v_i|^3 \goto 0.
$

We will prove the proposition in the following two steps:

\textbf{Step 1.} 

We need to prove that
\[
\sum_{i: |X_i| > 2 |v_i| } \frac{p(p-1)}{2}  |X_i|^{p-2} v_i^2
\]
converges to a constant. 

This can be computed similarly as before. By Lemma~\ref{lem:badregion} 
$$
\sum_{i: |X_i| < 2 |v_i| } \frac{p(p-1)}{2}  |X_i|^{p-2} v_i^2=o_{\P}(1)
$$
Therefore, we have
$$
\sum_{i: |X_i| > 2 |v_i| } \frac{p(p-1)}{2}  |X_i|^{p-2} v_i^2=\sum_{i } \frac{p(p-1)}{2}  |X_i|^{p-2} v_i^2+o_{\P}(1).
$$

We further obtain \begin{align*}
\sum_{i } \frac{p(p-1)}{2}  |X_i|^{p-2} v_i^2= \frac{p(p-1)}{2} \sum_{i}v_i^2\cdot (\frac{n}{\alpha})^{\frac{p-2}{\alpha}}\cdot |U_i|^{p-2}.
\end{align*}
Again, we also have \begin{align*}
 \sum_{i=1}^n v_i^2|\tilde X_i|^{p-2}\sim& (\frac{n}{p})^{\frac{p-2}{p}}\cdot \sum_{i=1}^nv_i^2 |\tilde U_i|^{p-2}\\
 =&(\frac{n}{p})^{\frac{p-2}{p}}\cdot \sum_{i=1}^n v_i^2| U_i|^{p-2}\cdot n^{\frac{p-2}{\alpha}-\frac{p-2}{p}}\\
 =&n^{\frac{p-2}{\alpha}}\cdot p^{\frac{2}{p}-1}\cdot \sum_{i=1}^nv_i^2 | U_i|^{p-2}.
 \end{align*}
 
 This implies $$
  \sum_{i=1}^n v_i^2| U_i|^{p-2}=n^{-\frac{p-2}{\alpha}}\cdot p^{1-\frac{2}{p}}\sum_{i }v_i^2 |\tilde X_i|^{p-2}\stackrel{p}{\to}n^{-\frac{p-2}{\alpha}}\cdot p^{1-\frac{2}{p}}\cdot\frac{\Gamma(1-\frac{1}{p})}{\Gamma(\frac{1}{p})}.
 $$
 Then we have  \begin{align*}
\sum_{i } \frac{p(p-1)}{2}  |X_i|^{p-2} v_i^2=& \frac{p(p-1)}{2}\cdot (\frac{n}{\alpha})^{\frac{p-2}{\alpha}} \sum_{i}v_i^2\cdot |U_i|^{p-2}\\
\stackrel{p}{\to}& \frac{1}{2}\cdot {\alpha}^{-\frac{p-2}{\alpha}}\cdot p^{3-\frac{2}{p}}\cdot\frac{\Gamma(2-\frac{1}{p})}{\Gamma(\frac{1}{p})}.
\end{align*}

\textbf{Step 2.} Show the third order term is vanishing

Prove that
\[
\sum_{i} |X_i|^{p-3} |v_i|^3 \goto 0.
\]
$$
\sum_i|X_i|^{p-3} |v_i|^3\le n\cdot(\sqrt\frac{\log n}{n})^3\cdot\max_i{|X_i|^{p-3}}
$$

According to Theorem~\ref{thm:uniform}, we have $$
X\stackrel{d}{=}t^{1/\alpha}\cdot r\cdot U,
$$
where $t\sim \Gamma(\frac{n}{\alpha}+1,1)$, $U$ has uniform distribution over $\partial K_\varphi=\{x:\|x\|_p=n^{1/p-1/\alpha} \}$,  $r$ has density $nx^{n-1}$ over $[0,1]$,  and $U,r,t$ are independent. 

As a result, by \Cref{cor:unif}
\begin{align*}
\max_i |X_i|\sim (\frac{n}{\alpha})^{\frac{1}{\alpha}}\cdot \max_{i} |U_i|.
 \end{align*}
 Consider $\tilde X\sim e^{-\|x\|_p^p}$, we then have $$
\tilde X\stackrel{d}{=}\tilde t\cdot r\cdot \tilde U,
$$
where $\tilde t\sim  \Gamma(\frac{n}{p}+1,1)$ and $\tilde U=n^{1/\alpha-1/p} U$.
 
Therefore, \begin{align*}
\max_i |\tilde X_i|\sim& (\frac{n}{p})^{\frac{1}{p}}\cdot \max_i |\tilde U_i|\\
 =&(\frac{n}{p})^{\frac{1}{p}}\cdot \max_{i} | U_i|\cdot n^{\frac{1}{\alpha}-\frac{1}{p}}\\
 =&n^{\frac{1}{\alpha}}\cdot p^{-\frac{1}{p}}\cdot \max_{i} | U_i|.
 \end{align*}
 Therefore, we have $$
 \max_i |U_i|\sim n^{-\frac{1}{\alpha}}\cdot p^{\frac{1}{p}}\max_i |\tilde X_i|\lesssim  n^{-\frac{1}{\alpha}}\cdot p^{\frac{1}{p}} (\log n)^{1/p},
 $$
 and we have 
$$ \max_i |X_i|\sim (\frac{n}{\alpha})^{\frac{1}{\alpha}}\cdot \max_{i} |U_i|\lesssim (\log n)^{1/p}
$$

As a result, 
$$
\sum_i|X_i|^{p-3} |v_i|^3\le n\cdot(\sqrt\frac{\log n}{n})^3\cdot\max_i{|X_i|^{p-3}}=(\log n)^{2.5-3/p}\cdot n^{-1/2}\to 0.
$$

\textbf{Verification of Condition \ref{d2}.} Use Sudakov's theorem to prove asymptotic normality.

Note that by Lemma~\ref{lem:badregion}
\[
\sum_{i: |X_i| > 2 |v_i| } p \sign(X_i) |X_i|^{p-1} v_i = o_{\P}(1) + \sum_{i=1}^n p \sign(X_i) |X_i|^{p-1} v_i.
\]

It suffices to show $\sum_{i=1}^n p \sign(X_i) |X_i|^{p-1} v_i$ is asymptotically a normal random variable. 

Firstly,  we have $$
 \sum_{i=1}^n (p \sign(X_i) |X_i|^{p-1} )^2=p^2 \sum_{i=1}^n |X_i|^{2(p-1)}
$$

According to Theorem~\ref{thm:uniform}, we have $$
X\stackrel{d}{=}t^{1/\alpha}\cdot r\cdot U,
$$
where $t\sim \Gamma(\frac{n}{\alpha}+1,1)$, $U$ has uniform distribution over $\partial K_\varphi=\{x:\|x\|_p=n^{1/p-1/\alpha} \}$,  $r$ has density $nx^{n-1}$ over $[0,1]$,  and $U,r,t$ are independent. 

As a result, by \Cref{cor:unif}
\begin{align*}
 \sum_{i=1}^n |X_i|^{2(p-1)}\sim (\frac{n}{\alpha})^{\frac{2p-2}{\alpha}}\cdot \sum_{i=1}^n |U_i|^{2p-2}.
 \end{align*}
 Consider $\tilde X\sim e^{-\|x\|_p^p}$, we then have $$
\tilde X\stackrel{d}{=}\tilde t\cdot r\cdot \tilde U,
$$
where $\tilde t\sim  \Gamma(\frac{n}{p}+1,1)$ and $\tilde U=n^{1/\alpha-1/p} U$.
 
Therefore, \begin{align*}
 \sum_{i=1}^n |\tilde X_i|^{2(p-1)}\sim& (\frac{n}{p})^{\frac{2p-2}{p}}\cdot \sum_{i=1}^n |\tilde U_i|^{2p-2}\\
 =&(\frac{n}{p})^{\frac{2p-2}{p}}\cdot \sum_{i=1}^n | U_i|^{2p-2}\cdot n^{\frac{2p-2}{\alpha}-\frac{2p-2}{p}}\\
 =&n^{\frac{2p-2}{\alpha}}\cdot p^{\frac{2}{p}-2}\cdot \sum_{i=1}^n | U_i|^{2p-2}.
 \end{align*}
 
 This implies $$
  \sum_{i=1}^n | U_i|^{2p-2}=n^{-\frac{2p-2}{\alpha}}\cdot p^{2-2/p} \sum_{i=1}^n |\tilde X_i|^{2(p-1)}\stackrel{p}{\to}p^{2-2/p}  n^{1-\frac{2p-2}{\alpha}}\cdot\E[|\tilde X_i|^{2p-2}],
 $$
 where
 $$
 \E[|\tilde X_i|^{2p-2}]=\frac{1}{\frac{2}{p}\Gamma(\frac{1}{p})}\int_{-\infty}^\infty |x|^{2p-2}e^{-|x|^p}\;dx=2\int_0^\infty x^{2p-2}e^{-x^p}\;dx=\frac{\Gamma(2-\frac{1}{p})}{\Gamma(\frac{1}{p})}
 $$
 is a constant when $p$ is of constant order.
 
 As a result, let $C=p^{4-2/p}\cdot\alpha^{-\frac{2p-2}{\alpha}}\cdot \frac{\Gamma(2-\frac{1}{p})}{\Gamma(\frac{1}{p})}$, we have
 \begin{align*}
 \sum_{i=1}^n (p \sign(X_i) |X_i|^{p-1} )^2=p^2 \sum_{i=1}^n |X_i|^{2(p-1)}\sim p^2(\frac{n}{\alpha})^{\frac{2p-2}{\alpha}}\cdot \sum_{i=1}^n |U_i|^{2p-2}\stackrel{p}{\to}C  n,
 \end{align*}
 satisfying the thin-shell condition of Sudakov's theorem and therefore $\sum_{i=1}^n p \sign(X_i) |X_i|^{p-1} v_i$ is asymptotically normal with variance $Cn$.

\subsection{Proof of Lemma 10.1}
According to Theorem~\ref{thm:uniform}, we have $$
X\stackrel{d}{=}t^{1/\alpha}\cdot r\cdot U,
$$
where $t\sim \Gamma(\frac{n}{\alpha}+1,1)$, $U$ has uniform distribution over $\partial K_\varphi=\{x:\|x\|_p=n^{1/p-1/\alpha} \}$,  $r$ has density $nx^{n-1}$ over $[0,1]$,  and $U,r,t$ are independent. 

As a result, $$
\|X\|_p=|t^{1/\alpha}|\cdot |r|\cdot n^{1/p-1/\alpha} \sim (\frac{1}{\alpha})^{1/\alpha}\cdot n^{1/p}
$$
\subsection{Proof of Lemma 10.2}
Denote $z_i=I(|X_i|\le 2|v_i|)=I(|t^{1/\alpha}\cdot r\cdot U_i|\le 2|v_i|)\le I(|t^{1/\alpha}\cdot r\cdot U_i|\le 2\sqrt{\frac{2\log n}{n}})$.

Consider $\tilde X\sim e^{-\|x\|_p^p}$, we then have $$
\tilde X\stackrel{d}{=}\tilde t^{1/p}\cdot r\cdot \tilde U,
$$
where $\tilde t\sim  \Gamma(\frac{n}{p}+1,1)$ and $\tilde U=n^{1/\alpha-1/p} U$.

Since $p/n=o(1)$, we then have $\tilde t\sim\frac{\alpha}{p}\cdot t$ and $\tilde t^{1/p}\sim \frac{n^{1/p-1/\alpha}\cdot\alpha^{1/\alpha}}{p^{1/p}}\cdot t^{1/\alpha}$, then 
\begin{align*}
\frac{1}{n}\sum_{i=1}^n I(|t^{1/\alpha}\cdot r\cdot U_i|\le 2\sqrt{\frac{2\log n}{n}})=&\frac{1}{n}\sum_{i=1}^n I(|\tilde t^{1/p}\cdot r\cdot \tilde U_i|\le 2 \frac{\alpha^{1/\alpha}}{p^{1/p}}\sqrt{\frac{2\log n}{n}})\\
=&\frac{1}{n}\sum_{i=1}^n I(|\tilde X_i|\le  2 \frac{n^{1/p-1/\alpha}\cdot\alpha^{1/\alpha}}{p^{1/p}}\sqrt{\frac{2\log n}{n}})\\
\sim& \P(|\tilde X_i|\le  2 \frac{ \alpha^{1/\alpha}}{p^{1/p}}\sqrt{\frac{2\log n}{n}}),
\end{align*}
where $\tilde X_i$ are $i.i.d.$ drawn from the population with density $\propto e^{-|x|^p}$.

When $|x|\le 2\frac{ \alpha^{1/\alpha}}{p^{1/p}}\sqrt{\frac{2\log n}{n}}$, and $2\frac{ \alpha^{1/\alpha}}{p^{1/p}}\sqrt{\frac{2\log n}{n}}=o(1)$, 
we then have $$
\P(|\tilde X_i|\le  2 \frac{ \alpha^{1/\alpha}}{p^{1/p}}\sqrt{\frac{2\log n}{n}})\asymp \frac{ \alpha^{1/\alpha}}{p^{1/p}}\sqrt{\frac{2\log n}{n}},
$$
which implies that $$
\frac{1}{n}\sum_{i=1}^n I(|X_i|\le 2|v_i|)\asymp  \frac{ \alpha^{1/\alpha}}{p^{1/p}} \sqrt{\frac{\log n}{n}}.
$$

As a result, we have that when $p>1$, $$
\sum_{i: |X_i| \le 2 |v_i| } ( |X_i + v_i|^p - |X_i|^p) \lesssim \frac{ \alpha^{1/\alpha}}{p^{1/p}}\sqrt{n\log n}\cdot (\frac{\log n}{\sqrt n})^p= \frac{ \alpha^{1/\alpha}}{p^{1/p}}n^{(1-p)/2}(\log n)^{p+1/2}=o(1).
$$

Similarly, for Lemma 5.4, we can use the same idea to show $$
\sum_{|X_i| < 2 |v_i| } p \sign(X_i) |X_i|^{p-1} v_i = o_{\P}(1).
$$
In fact, by using the same derivation, we have $$
\sum_{|X_i| < 2 |v_i| } p \sign(X_i) |X_i|^{p-1} v_i \lesssim \frac{ \alpha^{1/\alpha}}{p^{1/p}}\sqrt{n\log n}\cdot (\frac{\log n}{\sqrt n})^p= \frac{ \alpha^{1/\alpha}}{p^{1/p}}n^{(1-p)/2}(\log n)^{p+1/2}=o(1).
$$
\subsection{$p=1$}\label{pf:p1}

We now study the case where $p=1$. Since, for general $p$, we have \[
\tilde\varphi_n(x) = c_{p,\alpha} n^{1 - \frac{\alpha}{p}} \|x\|_p^{\alpha},
\]
where $c_{p,\alpha}=\alpha^{-1}\cdot p^{-\alpha+\frac{\alpha}{p}}\cdot\left(\frac{\Gamma(2-\frac{1}{p})}{\Gamma(\frac{1}{p})}\right)^{-\frac{\alpha}{2}}$.

Letting $p=1$ we get $$
\tilde\varphi_n(x)=\frac{1}{\alpha} n^{1-\alpha} \|x\|_1^\alpha.
$$

We first study the limit of $\|X\|_1$ when the density of $X$ is given by $\frac{1}{Z}e^{-\tilde\varphi_n(x)}$.

According to Theorem~\ref{thm:uniform}, we have $$
X\stackrel{d}{=}t^{1/\alpha}\cdot r\cdot U,
$$
where $t\sim \Gamma(\frac{n}{\alpha}+1,1)$, $U$ has uniform distribution over $\partial K_\varphi=\{x:\|x\|_1=(\alpha n^{\alpha-1})^{1/\alpha} \}$,  $r$ has density $nx^{n-1}$ over $[0,1]$,  and $U,r,t$ are independent. 

As a result, $$
\|X\|_1=|t^{1/\alpha}|\cdot |r|\cdot (\alpha n^{\alpha-1})^{1/\alpha} \sim (\frac{n}{\alpha})^{1/\alpha}\cdot 1\cdot(\alpha n^{\alpha-1})^{1/\alpha}= n.
$$
Now let us consider $\nabla \tilde\varphi_n(x)$, and we have $$
\nabla\tilde\varphi_n(X)=\frac{1}{\alpha} n^{1-\alpha} \cdot\alpha\cdot\|x\|_1^{\alpha-1}\cdot sgn(X)=sgn(X).
$$
Since $X$ is symmetric, the above expression implies the \textbf{Condition D3}, that is, $\|\nabla\tilde\varphi_n(X)\|_2\sim\sqrt{n}$ and therefore $v^\top\nabla\tilde\varphi_n(X) \to N(0,1)$.

To prove  \textbf{Condition D1}, since $\frac{1}{2}vI_{\varphi}v=\frac{1}{2}$, it suffices to show when $\|v\|_2=1, \alpha=1$, $$
\|X+v\|_1-\|X\|_1\to N(\frac{1}{2},1).
$$
The case where $\alpha> 1$ can be reduced to this setting by using the  following technique. Let us write $$
X\stackrel{d}{=}t^{1/\alpha}\cdot r\cdot U,
$$
where $t\sim \Gamma(\frac{n}{\alpha}+1,1)$, $U$ has uniform distribution over $\partial K_\varphi=\{x:\|x\|_p=n^{1-1/\alpha} \}$,  $r$ has density $nx^{n-1}$ over $[0,1]$,  and $U,r,t$ are independent. 
$$
\tilde X\stackrel{d}{=}\tilde t\cdot r\cdot \tilde U,
$$
where $\tilde t\sim  \Gamma({n}+1,1)$ and $\tilde U=n^{1/\alpha-1} U$.

Then \begin{align*}
\sum_{i}(|X_i+v_i|-|X_i|)\sim&t^{1/\alpha}\cdot r\cdot\sum_{i}{(|U_i+(\frac{n}{\alpha})^{-1/\alpha}v_i|-|U_i|)}\\
\sim&t^{1/\alpha}\cdot r\cdot n^{1-1/\alpha}\sum_{i}{(|\tilde U_i+n^{-1}\cdot \alpha^{1/\alpha} v_i|-|\tilde U_i|)}\\
\sim&(\frac{1}{\alpha})^{1/\alpha}\cdot\tilde t^{1/\alpha}\cdot r\cdot n^{1-1/\alpha}\sum_{i}{(|\tilde U_i+n^{-1}\cdot \alpha^{1/\alpha}v_i|-|\tilde U_i|)}\\
\sim&(\frac{1}{\alpha})^{1/\alpha}\cdot n^{1-1/\alpha}\sum_{i}{(|\tilde X_i+n^{1/\alpha-1}\cdot \alpha^{1/\alpha}v_i|-|\tilde X_i|)},
\end{align*}
which reduces to the $\alpha=1$ setting up to some scaling.

Therefore, it suffices to show the asymptotic normality of $\|X+v\|_1-\|X\|_1$ when $X$ has density $\propto e^{-\|X\|_1}$. We are going to use the  Berry-Esseen theorem. Suppose we have $n$ independent random variables $X_1,\ldots, X_n$ with $\E X_i = \mu_i, \Var X_i = \sigma_i^2, \E|X_i-\mu_i|^3 = \rho_i^3$. Consider the normalized random variable
$$S_n := \frac{\sum_{i=1}^n X_i-\mu_i}{\sqrt{\sum_{i=1}^n \sigma^2_i}}.$$
Denote its cdf by $F_n$. Then
\begin{theorem}[Berry-Esseen]\label{thm:BerryRV}
There exists a universal constant $C>0$ such that
	\[\sup_{x\in\R}|F_n(x)-\Phi(x)|\leqslant C\cdot \frac{\sum_{i=1}^n \rho_i^3}{\big(\sum_{i=1}^n\sigma_i^2\big)^{\frac{3}{2}}}.\]
\end{theorem}

In the following, we proceed to calculating these three moments for the random variable $X_1,...,X_n\sim p$ with density $p(x)=\frac{1}{2}e^{-|x|}$ for $x\in\R$. 

\subsubsection{Expectation}

Without loss of generality we assume $v>0$. 

\begin{align*}
&\E[|X+v|]-\E[|X|]=\frac{1}{2}\int_{-\infty}^\infty |x+v| e^{-|x|}\;dx-\frac{1}{2}\int_{-\infty}^\infty |x| e^{-|x|}\;dx\\
=&\frac{1}{2}(\int_{-\infty}^{-v} |x+v| e^{-|x|}\;dx+\int_{-v}^0 |x+v| e^{-|x|}\;dx+\int_{0}^\infty |x+v| e^{-|x|}\;dx-\int_{-\infty}^0 |x| e^{-|x|}\;dx-\int_{0}^\infty |x| e^{-|x|}\;dx)\\
=&\frac{1}{2}(\int_{-\infty}^{-v} -(x+v) e^{x}\;dx+\int_{-v}^0 (x+v) e^{x}\;dx+\int_{0}^\infty (x+v) e^{-x}\;dx+\int_{-\infty}^0 x e^{x}\;dx-\int_{0}^\infty x e^{-x}\;dx)\\
=&\frac{1}{2}(\int_{v}^\infty x e^{-x}\;dx-\int_{0}^v xe^{-x}\;dx-\int_{0}^\infty x e^{-x}\;dx+v(-\int_{v}^\infty e^{-x}\;dx+\int_0^v e^{-x}\;dx+\int_0^\infty e^{-x}\;dx))\\
=&-\int_{0}^v xe^{-x}\;dx+v\int_{0}^v e^{-x}\;dx\\
=&-(1-e^{-v}-ve^{-v})+v(1-e^{-v})\\
=&(v-1+e^{-v})= \frac{1}{2}v^2+o(v^2)
\end{align*}

\subsubsection{Variance}
\begin{align*}
2\E[|X+v|\cdot|X|]=&\int_{-\infty}^\infty |x+v|\cdot|x| e^{-|x|}\;dx\\
=&\int_{-\infty}^{-v} |x+v|\cdot|x| e^{-|x|}\;dx+\int_{-v}^0 |x+v|\cdot|x| e^{-|x|}\;dx+\int_{0}^\infty |x+v|\cdot|x| e^{-|x|}\;dx\\
=&\int_{-\infty}^{-v}(x+v)\cdot x e^{x}\;dx-\int_{-v}^0 (x+v)x e^{x}\;dx+\int_{0}^\infty (x+v)\cdot x e^{-x}\;dx\\
=&\int_{-\infty}^{-v} x^2e^{x}\;dx-\int_{-v}^0x^2e^{x}\;dx+\int_0^\infty x^2e^{-x}\;dx+v(\int_{-\infty}^{-v} xe^{x}\;dx-\int_{-v}^{0} xe^{x}\;dx+\int_0^\infty xe^{-x}\;dx)\\
=&2\int_{v}^{\infty} x^2e^{-x}\;dx+2v\cdot\int_{0}^{v} xe^{-x}\;dx\\
=&2v^2e^{-v}+4ve^{-v}+4e^{-v}+2v\cdot(1-ve^{-v}-e^{-v})\\
=&2v+2ve^{-v}+4e^{-v}\sim 4+o(v^2)
\end{align*}

\begin{align*}
2\E[|X+v|^2]=&\int_{-\infty}^\infty (x+v)^2 e^{-|x|}\;dx\\
=&\int_{-\infty}^\infty x^2 e^{-|x|}\;dx+v^2\int_{-\infty}^\infty e^{-|x|}\;dx\\
=&2\int_{0}^\infty x^2 e^{-x}\;dx+2v^2\int_{-\infty}^\infty e^{-x}\;dx\\
=&4+2v^2
\end{align*}

\begin{align*}
2\E[|X|^2]=&\int_{-\infty}^\infty |x|^2 e^{-|x|}\;dx\\
=&2\int_{0}^\infty x^2 e^{-x}\;dx\\
=&4
\end{align*}

\begin{align*}
2\E[(|X+v|-|X|)^2]=&\E[|X+v|^2]+\E[|X|^2]-2\E[|X+v|\cdot|X|]\\
=&\frac{1}{2}(8+2v^2-2(4+o(v^2)))\\
=&v^2+o(v^2)
\end{align*}
\subsubsection{Third moment}
\begin{align*}
\E[|X|^3]=&\int_{-\infty}^\infty |x|^3 e^{-|x|}\;dx\\
=&2\int_{0}^\infty x^3 e^{-x}\;dx\\
=&12
\end{align*}
\begin{align*}
2\E[|X+v|^3]=&\int_{-\infty}^\infty |x+v|^3 e^{-|x|}\;dx\\
=&\int_{-\infty}^{-v} |x+v|^3 e^{-|x|}\;dx+\int_{-v}^0 |x+v|^3 e^{-|x|}\;dx+\int_{0}^\infty |x+v|^3 e^{-|x|}\;dx\\
=&\int_{-\infty}^{-v} -(x+v)^3 e^{x}\;dx+\int_{-v}^0 (x+v)^3 e^{x}\;dx+\int_{0}^\infty (x+v)^3 e^{-x}\;dx\\
=&6e^{-v}+(-6 + 6 e^{-v} + 6 v - 3 v^2 + v^3)+(v^3+3v^2+6v+6)\\
=&12e^{-v}+12 v  + 2v^3
\end{align*}

\begin{align*}
2\E[|X+v|^2\cdot|X|]=&\int_{-\infty}^\infty (x+v)^2\cdot|x| e^{-|x|}\;dx\\
=&-\int_{-\infty}^{0} (x+v)^2 x e^{x}\;dx+\int_{0}^\infty (x+v)^2\cdot x e^{-x}\;dx\\
=&(v^2-4v+6)+(v^2+4v+6)\\
=&2v^2+12
\end{align*}

\begin{align*}
2\E[|X+v|\cdot|X|^2]=&\int_{-\infty}^\infty |x+v|\cdot x^2 e^{-|x|}\;dx\\
=&-\int_{-\infty}^{-v} (x+v)\cdot x^2 e^{x}\;dx+\int_{-v}^0 (x+v) x^2 e^{x}\;dx+\int_{0}^\infty (x+v) x^2 e^{-x}\;dx\\
=&e^{-v}(v^2+4v+6)+[2v-6+e^{-v}(v^2+4v+6)]+2(v+3)\\
=& 4v+2e^{-v}(v^2+4v+6)
\end{align*}

\begin{align*}
\E[(|X+v|-|X|)^3]=&\frac{1}{2}(-12+12e^{-v}+12 v  + 2v^3-3(2v^2+12)+3( 4v+2e^{-v}(v^2+4v+6)))\\
=&\frac{1}{2}(-48+48e^{-v}+24 v  + 2v^3-6v^2+6v^2e^{-v}+24ve^{-v})\\
=& \frac{1}{2}(48(-v+v^2/2+o(v^2))+24 v  + 2v^3-6v^2+6v^2(1-v+o(v))+24v(1-v+o(v)))\\
=& o(v^2) 
\end{align*}

Combining the pieces, we get $$
\sum_{i=1}^n\E[|X_i+v_i|-|X_i|]=\frac{1}{2}\sum_{i=1}^n v_i^2+o(1)=\frac{1}{2}+o(2)
$$
$$
\sum_{i=1}^n Var[|X_i+v_i|-|X_i|]=\sum_{i=1}^n v_i^2+o(1)=1+o(1)
$$
$$
\sum_{i=1}^n \E[(|X_i+v_i|-|X_i|-\E[|X_i+v_i|-|X_i|])^3]=o(\sum_{i=1}^n v_i^2)=o(1).
$$

Therefore, by using Theorem~\ref{thm:BerryRV}, we get the desired result. 



\end{appendices}
\end{document}